%% file: main_aistats2026_cameraready.tex
\documentclass[twoside]{article}

\usepackage[accepted]{aistats2026}
\usepackage[round]{natbib}

\renewcommand\paragraph[1]{%
  \noindent\textbf{#1}
}

\usepackage{amsmath, amsthm, amssymb}
\usepackage{mathtools}
\usepackage{algorithmic}
\usepackage{booktabs}
\usepackage{subcaption}
\input{math_commands}

\usepackage{nicefrac}
\usepackage{listings}
\usepackage{hyperref}
\usepackage{enumitem}
\usepackage{tikz}
\usepackage[capitalize,noabbrev,nameinlink]{cleveref}

\newcommand{\Section}[1]{\section{\MakeUppercase{#1}}} 

\newtheorem{theorem}{Theorem}[section]

\newtheorem{lemma}[theorem]{Lemma}
\newtheorem{proposition}[theorem]{Proposition}
\newtheorem{definition}[theorem]{Definition}

\newtheorem{remark}[theorem]{Remark}

\newcommand{\etatensor}{\boldsymbol{\eta}}
\newcommand{\Ytensor}{\boldsymbol{Y}}
\newcommand{\ii}{\mathbf{i}}          
\newcommand{\jj}{\mathbf{j}}          
\newcommand{\Real}{\mathbb{R}}        

\newcommand{\ie}{i.e., }

\newcommand{\eg}{e.g., }

\usetikzlibrary{patterns,patterns.meta}
\usetikzlibrary{shapes.geometric}
\usetikzlibrary{shapes.misc}
\usepackage{tikz-network}

\newcommand{\defaulttensorsize}{10pt}
\newcommand{\tensorsize}{\defaulttensorsize}
\tikzstyle{tensor}=[draw, inner sep=0, outer sep=0, minimum size=\tensorsize]
\tikzstyle{notensor}=[inner sep=0, outer xsep=2pt, outer ysep=0, minimum size=\tensorsize]
\tikzstyle{atensor}=[tensor, circle]
\tikzstyle{ctensor}=[tensor, circle]
\tikzstyle{dtensor}=[tensor, diamond]
\tikzstyle{wtensor}=[tensor]
\tikzstyle{ltensor}=[tensor, rounded rectangle, rounded rectangle left arc=none]
\tikzstyle{rtensor}=[tensor, rounded rectangle, rounded rectangle right arc=none]
\tikzstyle{etensor}=[tensor, minimum height=(1cm/\defaulttensorsize*0.5*2+1)*\tensorsize]
\tikzstyle{widetensor}[2]=[tensor, minimum width=(1cm/\defaulttensorsize*0.75*(#1-1)+1)*\tensorsize]
\newcommand{\stripesize}{4pt}
\pgfdeclarepatternformonly{stripes}{\pgfpointorigin}{\pgfpoint{\stripesize}{\stripesize}}{\pgfpoint{\stripesize}{\stripesize}}
{
    \pgfpathmoveto{\pgfpoint{0}{0}}
    \pgfpathlineto{\pgfpoint{\stripesize}{\stripesize}}
    \pgfpathlineto{\pgfpoint{\stripesize}{0.5*\stripesize}}
    \pgfpathlineto{\pgfpoint{0.5*\stripesize}{0}}
    \pgfpathclose%
    \pgfusepath{fill}
    \pgfpathmoveto{\pgfpoint{0}{0.5*\stripesize}}
    \pgfpathlineto{\pgfpoint{0}{\stripesize}}
    \pgfpathlineto{\pgfpoint{0.5*\stripesize}{\stripesize}}
    \pgfpathclose%
    \pgfusepath{fill}
}
\tikzstyle{striped}=[pattern=stripes, pattern color=lightgray]

\tikzstyle{tensornetwork}=[baseline=-0.25em, xscale=0.75, yscale=0.5,
                           every node/.style={inner sep=0, outer xsep=2pt, outer ysep=5pt, text depth=0pt},
                           execute at end picture={\path (current bounding box.south west) +(-2pt, 0) (current bounding box.north east) +(2pt, 0);}]

\usepackage{xspace}
\usepackage[colorinlistoftodos, color=blue!30!white]{todonotes}

\begin{document}

%

%
\runningauthor{Eliezer da Silva, Arto Klami, Diego Mesquita, Iñigo Urteaga}



\twocolumn[
\aistatstitle{On the Identifiability of Tensor Ranks via Prior Predictive Matching}
\aistatsauthor{Eliezer da Silva \And Arto Klami}
\aistatsaddress{
BCAM -- Basque Center for Applied Mathematics \\
University of Coimbra, CISUC/LASI, DEI 
\And 
Department of Computer Science
\\ University of Helsinki 
}

\aistatsauthor{Diego Mesquita \And Iñigo Urteaga}
\aistatsaddress{
School of Applied Mathematics \\ Getúlio Vargas Foundation \And
BCAM --- Basque Center for Applied Mathematics \\
IKERBASQUE --- Basque Foundation for Science
}
]

\begin{abstract}
Selecting the latent dimensions (ranks) in tensor factorization is a central challenge that often relies on heuristic methods.
This paper introduces a rigorous approach
to determine rank identifiability in probabilistic tensor models, based on prior predictive moment matching.
We transform a set of moment matching conditions
into a log-linear system of equations in terms of
marginal moments, prior hyperparameters, and ranks;
establishing an equivalence between rank identifiability and the solvability of such system.
We apply this framework to four foundational tensor-models,
demonstrating that the linear structure of the PARAFAC/CP model,
the chain structure of the Tensor Train model,
and the closed-loop structure of the Tensor Ring model yield solvable systems, making their ranks identifiable.
In contrast, we prove that the symmetric
topology of the Tucker model leads to an underdetermined system,
rendering the ranks unidentifiable by this method.
For the identifiable models, we derive explicit closed-form rank estimators based on the moments of observed data only.
We empirically validate these estimators and evaluate the robustness of the proposal.
\end{abstract}

\Section{Introduction}
\label{sec:intro}

Tensor factorizations have become indispensable tools for analyzing multi-way data across diverse contexts and fields such as neuroscience \citep{cong_eeg2015, DBLP:journals/bspc/MosayebiH20}, bioinformatics and genomics \citep{alter2003generalized}, chemometrics \citep{brochem2003}, computational social science \citep{DBLP:conf/kdd/ScheinPBW15},  network science \citep{contisciani2022inference}, recommender systems \citep{DBLP:conf/www/BhargavaPZL15, gopalan2015scalable}, supervised learning \citep{tensornet_supervised}, compression of large neural networks \citep{DBLP:conf/neurips/NovikovPOV15/tensorizing}, and  trace back to early work by \cite{hitchcock1927}. These methods decompose a high-dimensional tensor into a set of lower-dimensional factors, revealing latent structures and facilitating data interpretation and prediction \citep{DBLP:journals/jmlr/ChuG09, DBLP:journals/siamrev/KoldaB09, cichocki2009nonnegative}.

A fundamental challenge in the design and application of tensor models is the selection of their latent dimensions,
\ie the \textbf{ranks} ($r_p$), which dictate the model's complexity and its ability to capture the observed statistics of the data, without overfitting.
In contrast to simple matrix factorizations, most tensor models have multiple ranks, \ie $p\geq1$, depending on the specific form of the factorization.
The determination of the tensor ranks is, in general, a NP-Complete problem for tensors over $\mathbb{Q}$~\citep{DBLP:journals/jal/Hastad90}, NP-Hard for tensors over $\Real$~\citep{hillartensorhard2013} and hard to approximate in general \citep{DBLP:conf/approx/Swernofsky18}.

Bayesian approaches to tensor factorization~\citep{DBLP:journals/jmlr/ChuG09, DBLP:journals/pami/ZhaoZC15, DBLP:conf/icml/RaiWGCDC14} provide a principled alternative to exhaustive search, enabling uncertainty quantification, prior incorporation, and even automatic rank inference via non-parametric methods~\citep{DBLP:conf/aaai/PorteousBW08} or automatic relevance determination~\citep{DBLP:journals/pami/ZhaoZC15}.

However, parametric models often demand user-specified ranks or strong priors, complicating deployment in real-world scenarios where domain knowledge is scarce.
For matrix factorization models, the prior specification hurdle can be mitigated through prior predictive checks or empirical Bayes techniques~\citep{dasilva2023prior, DBLP:journals/jmlr/WangS21}.
That is, we can determine the matrix rank by comparing observed data with data simulated from the model under alternative ranks.
This established technique for matrix factorization inspires an analogous development for tensor factorization we study here.

The fundamental question we ask here is \emph{when can a prior predictive moment matching approach identify the ranks of tensor factorization models}.
To answer the question, we need to establish a range of tools for connecting the moments of the prior predictive distribution with the ranks of the tensor model and an approach to use them for characterizing the rank identifiability of the model.

In this work, we formulate a strategy for analyzing the identifiability of ranks in probabilistic tensor factorization models, by analyzing a log-linear system of equations derived from matching model-specific prior predictive moments with low-order moments of the observed tensor-data.
We show that the algebraic structure of this system, which is determined by the chosen tensor decomposition, dictates whether the ranks can be uniquely identified by the method of moment-matching.
We formalize this analysis by transforming the moment equations into a log-linear system, where identifiability is equivalent to the solvability of the system with respect to the unknown ranks.

Fundamentally, this framework establishes a link between a tensor factorization model's structure and its rank identifiability.
We provide a clear theoretical characterization and distinction between different tensor models,
demonstrating that rank identifiability is a direct consequence of the interaction topology of their latent factors.
In addition, we showcase how ranks can be empirically estimated from observed tensor-data according to the prior predictive matching method.

Our specific contributions are:
\begin{enumerate}[topsep=0pt, left=4pt]
    \item We establish a general, prior predictive based framework for analyzing rank identifiability in probabilistic tensor models, by examining the rank of a log-linear system of moment equations.
    \vspace{-0.1cm}
    \item We apply this framework to four foundational tensor decompositions:
    \textbf{Tucker}~\citep{tucker1966some_tucker},    
    \textbf{PARAFAC/CP}~\citep{hitchcock1927, harshman1970foundations_parafac, cpcande70},
    \textbf{Tensor Train}~\citep[\textbf{TT},][]{oseledets2011tensor_tt},
    and \textbf{Tensor Ring}~\citep[\textbf{TR},][]{zhao2016_tr}. \looseness=-1
    \item We prove that the ranks of the standard Tucker model are not identifiable with this method due to symmetries in its moment structure leading to a degenerate, unsolvable system of equations.
    \item We prove that the ranks of the PARAFAC/CP, TT and TR models are identifiable from the first and second moments and derive explicit closed-form estimators for them.
    \item We present and evaluate a complete and robust pipeline for the estimation of ranks from observed tensor-data in identifiable models.
\end{enumerate}



\Section{A General Framework for Rank Identifiability}
\label{sec:framework}

We hereby introduce the notation and assumptions to establish a general framework for rank identifiability in probabilistic tensor factorization models.

\paragraph{Notation.}
Let $\Ytensor \in \mathbb{R}^{N_1 \times \dots \times N_M}$ be an observed tensor with $M$ modes. 
A \emph{mode} of an order-$M$ tensor is one of its $M$ dimensions. Let $\mathbf{i} = (i_m)_{m \in [M]} \in \bigtimes_{m=1}^M[N_m]$ be the multi-index from the observed tensor domain, where $[N_m] = \{1, \dots, N_m\}$ . Let $ \mathbf{r}=( r_p )_{p \in [R]} $ be the multi-rank vector of size $R \leq M$. 
Let $\beta \in \mathcal K = \bigtimes_{p=1}^{R} [r_p]$ be the latent multi-index over the rank domains.
For multi-indices $\mathbf{i}, \mathbf{j}$, the shared-mode set is $S(\mathbf{i}, \mathbf{j}) = \{k : i_k = j_k\}$. We use bold sans-serif for tensors ($\etatensor$) and vectors/matrices ($\boldsymbol{\theta}^{(p)}$).
All latent factor multi-indices $\mathbf{i}_p(\mathbf{i}, \beta)$ are formed by combining some components of $\mathbf{i}$ and the latent index $\beta$, and each factor is probabilistically defined as $\theta^{(p)}_{\mathbf{i}_p(\mathbf{i}, \beta)} \sim \pi_p(\mu_p, \sigma_p^2)$,
i.i.d. across indices with $p \in \mathcal{P}$  indexing and grouping the factors. 

\begin{definition}[The Probabilistic Tensor-Model] \label{def:setting}
We assume that elements in tensor $\Ytensor \in \mathbb{R}^{N_1 \times \dots \times N_M}$ are probabilistically generated from a rate tensor $\etatensor \in \mathbb{R}^{N_1 \times \dots \times N_M}$,
whose structure is defined by a (model-dependent) tensor decomposition with latent factors of ranks $ \mathbf{r}=( r_1, \ldots , r_R )$ of the following generic form:

\begin{equation} \label{eq:general:factorization}
    \eta_{\mathbf{i}} = \sum_{\beta \in \mathcal K} \prod_{p \in \mathcal{P}} \theta^{(p)}_{\mathbf{i}_p(\mathbf{i}, \beta)},
\end{equation}
with latent indices defined by $\mathbf{i}_p(\mathbf{i}, \beta)$ and parameters $\theta^{(p)}_{\mathbf{i}_p(\mathbf{i}, \beta)} \sim \pi_p(\mu_p, \sigma_p^2)$, drawn independently from a factor $p$-specific location-scale prior distribution $\pi_p(\cdot)$ with up-to-second order moments $(\mu_p, \sigma_p^2)$.

Tensor observations $Y_{\mathbf{i}}$ are conditionally independent given the rate tensor $\etatensor$, with conditional moments $\E[Y_{\mathbf{i}} | \etatensor] = \eta_{\mathbf{i}}$ and $\Var(Y_{\mathbf{i}} | \etatensor) = \phi(\eta_{\mathbf{i}})$ for some model-dependent positive function $\phi(\cdot)$.
\end{definition}

This probabilistic tensor formulation allows for different observation models (e.g. Poisson with $\phi(\eta)=\eta$, or Gaussian with $\phi(\eta)=\sigma_Y^2$ variance),
and it generalizes existing matrix factorization frameworks \citep{dasilva2023prior}, setting the stage for the prior predictive moment matching analysis.
Commonly used factorization models such as Tucker, PARAFAC/CP, TR and TT, are all instances of this general formulation \footnote{If $\mathbf{i}_p(\mathbf{i}, \beta)$ is formed by taking the $p$-th components of $\mathbf{i}$ and $\beta$ respectively, for $p>0$, and for $p=0$ define a special $\mathbf{i}_0(\mathbf{i}, \beta)=\beta$, we obtain Tucker. While doing the same but defining $\beta$ as a one dimensional index shared by all, without a special $p=0$ index, we obtain PARAFAC/CP.}.

\paragraph{Observable data and moments of interest.}
The second moments of the observable tensor-data can be described by two key quantities:
the total covariance and the pure interaction terms, which form the key building blocks for the analysis that follows.

\begin{definition}[Total Covariance and Pure Interaction Terms]
\label{def:covariances}
Let $\mathbf{i} = (i_1, \dots, i_M)$ and $\mathbf{j} = (j_1, \dots, j_M)$ be two multi-indices from the tensor's domain. Let $S(\mathbf{i}, \mathbf{j}) = \{k \mid i_k = j_k\}$ be the set of indices they share.

The \textbf{total observable covariance}, $C_S$, is the theoretical covariance between two multi-index tensor entries $\mathbf{i}$ and $\mathbf{j}$:
    $$C_S = \text{Cov}(Y_{\mathbf{i}}, Y_{\mathbf{j}}) \;.$$
    Given a factorization structure, this value depends only on the set of shared indices $S(\mathbf{i}, \mathbf{j})$. For entries sharing no indices, we assume $C_{\emptyset}=0$.

The \textbf{pure interaction term}, $v_S$, represents the variance arising uniquely from the interaction of factors in $S(\mathbf{i}, \mathbf{j})$,
    and induces the definition of the total covariance in terms of their additive contribution over the non-empty subsets of $S(\mathbf{i}, \mathbf{j})$:
    $$C_S = \sum_{\emptyset \neq S' \subseteq S} v_{S'} \;.$$
    These pure interaction terms can be also defined via the principle of inclusion-exclusion from the observable covariances:
    $$v_S = \sum_{S' \subseteq S} (-1)^{|S|-|S'|} C_{S'}.$$

\end{definition}

\paragraph{Connecting theoretical and observable moments.}
For many tensor decomposition models,
the theoretical moments of the rate tensor $\etatensor$ can be expressed as polynomials:
\ie as products of the unknown parameters
(ranks and prior moments).

\begin{proposition}[Polynomial Structure of Rate Moments]
\label{prop:polynomial_moments}
For a probabilistic tensor model congruent with Definition~\ref{def:setting},
every pure interaction term $v_S$ is a monomial in the unknown model hyperparameters ($r_p, \mu_p, \sigma^2_p$).
Therefore, the marginal mean $\E[\etatensor]$,
the variance $\Var(\etatensor)$,
and the covariance terms $\Cov(\eta_{\mathbf{i}}, \eta_{\mathbf{j}})$ can be expressed as polynomials of $\{r_p, \mu_p, \sigma_p^2\}$.
\end{proposition}

\begin{proof}
Each rate element $\eta_{\mathbf{i}}$ is a sum-product of the latent factor elements $\{ \theta^{(p)}_{\mathbf{i}_p(\mathbf{i}, \beta)} \}$. Applying total expectation we obtain $\E[\Ytensor_{\mathbf{i}}]= \sum_{\beta \in \mathcal K}\underset{p}{\prod}\mu_p$, with the summation index depending on the rank dimensions, and resulting in the monomial given by products of ranks $\{r_p\}$ and means $\{\mu_p\}$. Applying total covariance and conditional independence,
the computation leads to moments $\E[\eta_{\mathbf{i}}\eta_{\mathbf{j}}]$,
which resolve into a double sum-product over latent indexes,
with $\sigma_p^2$ appearing for factors with the same latent index,
$\mu_p^2$ for independent factors,
and powers of $r_p$ from interactions of repeated sums and multiplications,
leading to a polynomial on prior parameters $\{\mu_p, \sigma_p^2\}$ and ranks $\{r_p\}$.
\end{proof}


\subsection{The Log-Linear System of Prior Predictive Moment Equations}
The insight unveiled by Proposition \ref{prop:polynomial_moments} is 
that observable moments (like covariances $C_S$) are polynomials of the model parameters. 
More precisely, the individual terms of these polynomials are the \textbf{pure interaction terms} $v_S$, which are \textbf{monomials}.
This monomial (multiplicative) structure is ideal for linearization.
Namely, by taking the logarithm of the equation for each pure interaction term,
we transform the system of prior predictive moments into a linear one, amenable to standard algebraic analysis.

\begin{definition}[The Log-Linear System of Rate Moments]
\label{def:log_linear_system}
The log-linear system derived from the monomial equations of  pure interaction terms is given by $\mathbf{A} \mathbf{x} = \mathbf{b}$, where $\mathbf{x}$ is the vector of the logarithms of the unknown parameters,
\ie $\log(r_p), \log(\mu_p^2), \log(\sigma_p^2)$;  $\mathbf{b}$ is the vector of the logarithms of the (estimable) moment terms,
\ie $\log(E[Y]), \log(v_S)$;
and $\mathbf{A}$ is the integer-valued \textbf{design matrix}, where each entry $A_{ij}$ is the exponent of the $j$-th unknown parameter in the monomial expression for the $i$-th moment term.
\end{definition}

While the invertibility of the full matrix $\mathbf{A}$ determines if \textit{all} parameters are identifiable, our primary interest is in the ranks of the tensors. The identifiability of the ranks hinges on whether they can be isolated from the other unknown parameters.
This leads to a more precise, and relevant,
principle of rank identifiability:

\begin{proposition}[Principle of Rank Identifiability]
\label{prop:identifiability}
The ranks $\{r_p\}$ of a tensor model are identifiable from the second moments of tensor-data if and only if the unknown prior-moment parameters can be algebraically eliminated from the full log-linear system of moments $\mathbf{A}\mathbf{x}=\mathbf{b}$.

Consequently, it yields a non-trivial \textbf{reduced system} of the log-ranks $\mathbf{x}_r$ alone:
\ie $\mathbf{A}_{\text{red}} \mathbf{x}_r = \mathbf{b}_{\text{red}}$,
where the reduced design matrix $\mathbf{A}_{\text{red}}$ has full rank.
A full-rank $\mathbf{A}_{\text{red}}$ guarantees a unique solution for the ranks, meaning they occupy a parameter subspace that can be uniquely determined from the data's moments.
\end{proposition}

Propositions \ref{prop:polynomial_moments} and \ref{prop:identifiability} establish a direct link between the model-dependent, prior predictive moments and the moments of the observed tensor-data.
This connection is formalized through an identifiability analysis of a reduced, rank-specific log-linear system.

However, we note that the utility of this link for rank estimation is fully model-dependent.
We first present in Section~\ref{sec:tucker_analysis} a cautionary example where the dependencies between prior hyperparameters and observed second moments are not identifiable for the popular Tucker tensor model.
In contrast, we highlight in Section~\ref{sec:rank_identification} several models where identifiability holds, enabling the design and implementation of a novel method for estimating tensor ranks directly and exclusively from observed moments.

\Section{A Cautionary Tale: The Tucker Model}
\label{sec:tucker_analysis}

We now examine the standard Tucker decomposition and show that its symmetric ``hub-and-spoke'' \cite{hubsspokenetwork} interaction topology precludes forming a solvable reduced system of equations for the Tucker model's ranks.

\subsection{Model Definition and Its Moments}

\begin{definition}[The Tucker model]
For an order-$M$ tensor, with ranks $\mathbf{r}=(r_1, \dots, r_M)$ the rate $\etatensor$ is defined by the Tucker decomposition as
$$\eta_{i_1 \dots i_M} = \sum_{k_1=1}^{r_1} \dots \sum_{k_M=1}^{r_M} G_{k_1 \dots k_M} \prod_{p=1}^M \theta^{(p)}_{i_p, k_p} \;.$$
\end{definition}
This model, introduced by~\citet{tucker1966some_tucker}, is widely used for its interpretability and versatility in multi-way analysis.

The ranks to be estimated $(r_1, \dots, r_M)$ are the dimensions of the core tensor $G$. As per Definition~\ref{def:setting} the elements of $G$ and each factor $\theta^{(p)}$ are i.i.d. random variables with moments $(\mu_G, \sigma_G^2)$ and $(\mu_p, \sigma_p^2)$, respectively.

\begin{lemma}[Moment Structure of the Tucker Model]
\label{lemma:tucker_moments}
For the order-M Tucker model, the squared mean and the pure interaction terms involving the core ($v_G$), a single factor ($v_p$), or both ($v_{G,p}$) obey
---see detailed derivations in Appendix~\ref{app:tucker}:
\begin{align*}
    (E[Y])^2 &= \mu_G^2 \prod_{q=1}^M (r_q^2 \mu_q^2) \;,\\
    v_G &= \sigma_G^2 \prod_{q=1}^M (r_q^2 \mu_q^2) \;,\\
    v_p &= \mu_G^2 (r_p \sigma_p^2) \prod_{q \neq p} (r_q^2 \mu_q^2) \;,\\
    v_{G,p} &= \sigma_G^2 (r_p \sigma_p^2) \prod_{q \neq p} (r_q^2 \mu_q^2) \;.
\end{align*}
\end{lemma}

\subsection{Rank Identifiability Analysis}

We now prove that the symmetric topology of the Tucker model leads to a degeneracy in its moment equations
by showing that the reduced design matrix for the ranks is the zero matrix.
Hence, this structure makes the ranks fundamentally unrecoverable from observed tensor-data.

\begin{lemma}[An Identity Among Observables]\label{lem:tucker_identity}
In the log-linear system for the Tucker model moments,
the following $M$ independent relations hold for every mode $p = 1,\dots,M$:
\begin{equation}\label{eq:tucker_id}
    \log v_p - 2 \log E[Y] = \log v_{G,p} - \log v_G \;.
\end{equation}
\end{lemma}
\begin{proof}
We show that both sides of the equation reduce to the same expression. Using the monomials from Lemma~\ref{lemma:tucker_moments}, the left-hand side is:
\begin{align*}
    \log v_p - 2 \log E[Y] &= \log\left[\mu_G^2 r_p \sigma_p^2 \prod_{q\neq p}(r_q^2 \mu_q^2)\right] \\
    &- 2 \log\left[\mu_G \prod_{q=1}^{M}(r_q \mu_q)\right] \\
    &= \log\left( \frac{\mu_G^2 r_p \sigma_p^2 \prod_{q\neq p}(r_q^2 \mu_q^2)}{\mu_G^2 \prod_{q=1}^{M}(r_q^2 \mu_q^2)} \right) \\
    &= \log\left(\frac{1}{r_p}\frac{\sigma_p^2}{\mu_p^2}\right) \;.
\end{align*}
An identical calculation shows that $\log v_{G,p} - \log v_G$ also simplifies to $\log\left(\frac{1}{r_p}\frac{\sigma_p^2}{\mu_p^2}\right)$, demonstrating the identity in Equation~\ref{eq:tucker_id}.
\end{proof}

\begin{theorem}[Non-identifiability of Tucker ranks]\label{thm:tucker:complete}
The Tucker model ranks $(r_1, \dots, r_M)$ are
\textbf{not identifiable} directly from the second moments of the tensor-data.
\end{theorem}
\begin{proof}
We prove that the reduced design matrix $\mathbf{A}_{\text{red}}$ is the zero matrix.

The identity established in Lemma~\ref{lem:tucker_identity} provides $M$ linear dependencies among the rows of the full design matrix $\mathbf{A}$. Because these $M$ identities express every standard-basis vector $\mathbf{e}_p$ (the column that selects $\log r_p$) as a linear combination of columns that are free of any log-rank variable, the $M$ columns corresponding to the log-ranks are redundant. Consequently, the reduced design matrix $\mathbf{A}_{\text{red}}$ of Proposition~\ref{prop:identifiability} is the $0 \times M$ zero matrix, and the ranks are unidentifiable.
\end{proof}



\Section{exemplary tales: Identifiable ranks with closed-form expression}\label{sec:rank_identification}

In contrast to the analysis in the previous section, we present below the general solutions for the ranks of the PARAFAC/CP, Tensor Train, and Tensor Ring models for an arbitrary order $M$ tensor, showcasing how to establish identifiability according to the principles formalized in Proposition~\ref{prop:identifiability}.

\subsection{PARAFAC/CP Model}\label{sec:parafac:main}

The PARAFAC or Canonical Polyadic (CP) model's rank is identifiable due to a unique property of its second moments: the additivity of its variance components.
Precisely, its total observable covariance decomposes into a clean sum of pure, monomial interaction terms, because the model is defined as a sum of $r$ independent rank-one tensors.
This structure provides the necessary constraints for a solvable rank system.

\subsubsection{Model Definition and Its Moments}
\begin{definition}[The PARAFAC/CP model]
For an order-$M$ tensor, the rate $\etatensor$ of a PARAFAC/CP decomposition of rank $r$ follows:
\vspace{-0.1cm}
$$\eta_{\mathbf{i}} = \eta_{i_1 \dots i_M} = \sum_{k=1}^r \prod_{p=1}^M \theta^{(p)}_{i_p,k} \;.$$
\end{definition}
\vspace{-0.2cm}

The PARAFAC/CP decomposition, originally proposed by  \cite{hitchcock1927} , with later formulations by
\citet{harshman1970foundations_parafac} and \citet{cpcande70}, assumes a sum of rank-1 tensors and is widely used in multiple fields given its simple structure and applicability. The variance calculation of the CP model results in a simple and regular form for its pure interaction terms, as formalized below, a fundamental property leading to a closed-form expression for the rank $r$.

\begin{lemma}[Monomial Moment structure of PARAFAC/CP]\label{lem:CP-monomials}
For an order-M PARAFAC/CP model with rank $r$ and factor moments $(\mu_k, \sigma_k^2)$, the squared mean and the pure-interaction terms obey
\begin{align*}
(E[Y])^2 &= r^2\prod_{k=1}^{M}\mu_k^2 \;,\\[1mm]
v_{\{p\}} &= r\sigma_p^2\prod_{k\neq p}\mu_k^2 \;,\\[1mm]
v_{\{p,q\}} &= r\sigma_p^2\sigma_q^2\prod_{k\neq p,q}\mu_k^2 \qquad(p\neq q) \;.
\end{align*}
\end{lemma}
\begin{proof}
The rate is $\eta_{\mathbf{i}} = \sum_{s=1}^r \prod_{k=1}^M \theta^{(k)}_{i_k,s}$. Because all factor entries across different components $s$ are independent, the variance of the sum is the sum of the variances. The pure interaction term $v_S$ is the component of the total variance arising from the variance of factors in the set $S$ and the mean of factors not in $S$. For $v_{\{p\}}$, this is:
\begin{align*}
     v_{\{p\}} &= \mathrm{Var}_{\text{pure}}(\eta_\mathbf{i} \mid \text{mode } p) \\
     &= \sum_{s=1}^r \mathrm{Var}(\theta^{(p)}_{i_p,s}) \prod_{k \neq p} \left(E[\theta^{(k)}_{i_k,s}]\right)^2 \\
     &= r\sigma_p^2\prod_{k \neq p}\mu_k^2 \;.
\end{align*}

The derivations for other pure terms follow the same principle of summing variances over $r$ independent components.
\end{proof}

\subsubsection{Rank Identifiability Analysis}
The simple monomial structure of the moments leads directly to a closed-form solution for the rank.

\begin{theorem}[Closed-Form Rank Estimator for PARAFAC/CP]
For any two distinct modes $p \neq q$, the rank $r$ is identifiable and given by:
\begin{equation}\label{eq:parafac:closed}
    r = \frac{v_{\{p,q\}} \cdot (E[Y])^2}{v_{\{p\}} \cdot v_{\{q\}}} \;.
\end{equation}
\end{theorem}
\begin{proof}
We substitute the monomials from Lemma~\ref{lem:CP-monomials} into the expression:
\begin{align*}
\frac{v_{\{p,q\}}\cdot(E[Y])^{2}}{v_{\{p\}}\cdot v_{\{q\}}}
&=\frac{\bigl[r\sigma_p^2\sigma_q^2\prod_{k\neq p,q}\mu_k^2\bigr]\cdot\bigl[r^2\prod_{k=1}^{M}\mu_k^2\bigr]}{\bigl[r\sigma_p^2\prod_{k\neq p}\mu_k^2\bigr]\cdot\bigl[r\sigma_q^2\prod_{k\neq q}\mu_k^2\bigr]}\\[2mm]
&=\frac{r^3 \left(\prod_{k\neq p,q}\mu_k^2\right) \left(\mu_p^2\mu_q^2\prod_{k\neq p,q}\mu_k^2\right) }{r^2\left(\mu_q^2\prod_{k\neq p,q}\mu_k^2\right)\cdot\left(\mu_p^2\prod_{k\neq p,q}\mu_k^2\right)} \\
&= r.
\end{align*}
Every nuisance parameter cancels, yielding the rank itself. Hence, $r$ is identifiable and the formula is exact.
\end{proof}

\subsection{The Tensor Train Model}\label{sec:tt:main}

The Tensor Train (TT) model's topology is an open, asymmetric chain of interactions \citep{oseledets2011tensor_tt}.
This sequential structure breaks the symmetries that cause non-identifiability in the Tucker model,
providing a sufficient number of distinct algebraic constraints to solve for all of the model's internal ranks.

\subsubsection{Model Definition and Its Moments}

\begin{definition}[The Tensor-Train model]
For an order-$M$ tensor, the rate $\etatensor$ of a Tensor Train decomposition with TT-ranks $(r_1, \dots, r_{M-1})$ follows
$$\eta_{i_1 \dots i_M} = \sum_{k_1=1}^{r_1} \dots \sum_{k_{M-1}=1}^{r_{M-1}} \theta^{(1)}_{1, i_1, k_1} \theta^{(2)}_{k_1, i_2, k_2} \cdots \theta^{(M)}_{k_{M-1}, i_M, 1} .$$
\end{definition}

The ranks to be estimated are the internal bond dimensions $(r_1, \dots, r_{M-1})$,
which we refer to as \textit{interior ranks}.
A full derivation of the general TT moments involves a series of (repeated) algebraic calculations, which we present in Appendix~\ref{app:tensor_train}.

However, their algebraic structure contains the necessary asymmetry for rank identification. The key is that pure interaction terms involving non-contiguous modes (e.g., $\{p-1, p+2\}$) have a different dependence on the ranks than terms involving contiguous modes (e.g., $\{p, p+1\}$).

\begin{lemma}[Monomial Moment Structure of Tensor-Train]\label{lem:TT-monomials}
The pure-interaction monomials for the TT model are functions of the ranks and prior moments. Their exact form depends on the chosen modes in $S$,
but they all exhibit a crucial asymmetry.
For an interior rank $r_p$,
the relevant monomials for its identification are:
\begin{align*}
v_{{p+1}}&=r_{p}r_{p+1}r_{p+2}\mu_p^2\sigma_{p+1}^2\mu_{p+2}^2\prod_{k\neq p,p+1,p+2}\mu_k^2 \;,\\
v_{{p,p+1}}&=r_{p}r_{p+1}r_{p+2}\sigma_p^2\sigma_{p+1}^2\mu_{p+2}^2\prod_{k\neq p,p+1,p+2}\mu_k^2 \;,\\
v_{{p-1,p+2}}&=r_{p-1}r_{p}r_{p+1}r_{p+2}\sigma_{p-1}^2\mu_p^2\sigma_{p+2}^2\prod_{k\neq p-1,p,p+2}\mu_k^2 \;,\\
v_{{p-1,p,p+2}}&=r_{p-1}r_{p}r_{p+1}r_{p+2}\sigma_{p-1}^2\sigma_p^2\sigma_{p+2}^2\prod_{k\neq p-1,p,p+2}\mu_k^2 \;.
\end{align*}

Notably, the following identities hold:
\begin{align}
    \frac{v_{\{p,p+1\}}}{v_{\{p+1\}}} &= \frac{\sigma_p^2}{\mu_p^2} \;, \\
    \frac{v_{\{p-1,p,p+2\}}}{v_{\{p-1,p+2\}}} &= \frac{1}{r_p}\frac{\sigma_p^2}{\mu_p^2} \;.
\end{align}
\end{lemma}

\begin{proof}[Informal Proof]
The TT rate is a sum over products of core tensor entries. Each pure-interaction term is obtained by calculating the covariance, which involves a double summation over all internal rank indices. The exponents of the ranks in the final monomial count how many of these summations remain independent after constraints from the interaction set are applied. The chain-like structure ensures that non-contiguous interactions constrain the summations differently than contiguous ones, producing the necessary asymmetry
---see Appendix~\ref{app:tensor_train}
for detailed derivations.
\end{proof}

\subsubsection{Rank Identifiability Analysis}
The asymmetry of the moment structure allows us to construct a solvable system for each interior rank.

\begin{theorem}[Closed-Form Interior-Rank Estimator for TT]
For any interior rank $1 < p < M-1$, $r_p$ is identifiable and given by
\begin{equation}\label{eq:tt:closed}
    r_p = \frac{v_{\{p,p+1\}} \cdot v_{\{p-1,p+2\}}}{v_{\{p+1\}} \cdot v_{\{p-1,p,p+2\}}} \;.
\end{equation}
\end{theorem}
\begin{proof}
The proof follows by dividing the two identities from Lemma~\ref{lem:TT-monomials}.
The unknown coefficient of variation term, $\sigma_p^2/\mu_p^2$, cancels, resulting in the rank equation:
$$ r_p = \frac{v_{\{p,p+1\}}/v_{\{p+1\}}}{v_{\{p-1,p,p+2\}}/v_{\{p-1,p+2\}}} = \frac{v_{\{p,p+1\}} \cdot v_{\{p-1,p+2\}}}{v_{\{p+1\}} \cdot v_{\{p-1,p,p+2\}}} .$$
Since the rank can be expressed purely in terms of observable moments, it is identifiable.
\end{proof}

\begin{remark}[The Tensor Ring Model]
A generalization of the TT, known as the Tensor Ring (TR) model, can be formed by starting from the TT and connecting its beginning and end, creating a closed loop with no limitations on the outer rank dimension.
We prove the identifiability of such TR model, and present its corresponding rank estimation formulas in Appendix ~\ref{app:tensor_ring}.
\end{remark}

\Section{The Tensor-Rank Estimation Pipeline}
\label{sec:formal_pipeline}

We now present a complete, algorithmically specified, robust pipeline for applying the theoretical results of this work:
it consists of estimating from observed tensor-data
the pure-interaction terms $\{\hat{v}_S\}$ 
and the estimated mean $\hat{E}[Y]$,
to replace them in the specific closed-form expression
derived in Section~\ref{sec:rank_identification} for the ranks of each identifiable model. 

Recall that all rank estimation expressions are written as a division,
\ie the rank $r_p$ is always calculated via an equation of the $\frac{\text{Num}_p}{\text{Den}_p}$ form, with $\text{Num}_p$ the numerator and $\text{Den}_p$ the denominator;
where expressions are adjusted for each specific model.

Our tensor-rank estimation procedure is based on
(1) computing an unbiased, single-pass estimation of the aforementioned moments;
(2) a bootstrap procedure to provide a distribution of empirical (regularized) rank estimates;
and (3) a summarization step to yield robust point and uncertainty-aware rank estimates.

\paragraph{Inner loop: \textit{The Unbiased Moment Estimation}.}
The fundamental step is to compute unbiased estimates of pure interaction terms $\{\hat{v}_S\}$ from the observed tensor-data $\mathcal{Y}$:
\begin{enumerate}[noitemsep, topsep=0pt, left=4pt]
    \item \textbf{The Global Mean:} Compute the empirical global mean of the tensor, $\hat{E}[Y] = |\mathcal{Y}|^{-1} \sum_{\mathbf{i}} Y_\mathbf{i}$.
    \item \textbf{Total Covariances:} For every non-empty pattern of shared indices $S \subseteq \{1, \dots, M\}$, compute the sample covariance $\hat{C}_S$ by drawing a (large) number of pairs of entries $(\mathbf{a}_j, \mathbf{b}_j)$ that match exactly on the index set $S$, \ie
    \[ \hat{C}_S = (N_{\text{cov}} - 1)^{-1} \sum_{j=1}^{N_{\text{cov}}} (Y_{\mathbf{a}_j} - \hat{E}[Y])(Y_{\mathbf{b}_j} - \hat{E}[Y]) \;. \]
    \item \textbf{Pure Interactions:} Calculate the pure interaction terms from the total covariances using the exact inclusion-exclusion principle:
    \[ \hat{v}_S = \sum_{S^\prime \subseteq S} (-1)^{|S|-|S^\prime|} \hat{C}_{S^\prime} \;.\]
\end{enumerate}

\paragraph{Outer loop: \textit{The Bootstrapped Distribution of Rank Estimates}.}
To ensure robustness and quantify the estimation uncertainty, we repeat the inner moment estimation procedure described above on bootstrap replicates of the observed tensor-data.

For $b = 1, \dots, B$:
\begin{enumerate}[noitemsep, topsep=0pt, left=4pt]
    \item A bootstrap sample $\Ytensor^{(b)}$ is generated from the original tensor $\Ytensor$, using a block bootstrap on tensor slices to preserve potential dependencies.
    \item The complete ``\textit{Inner loop}'' estimation procedure is applied to $\Ytensor^{(b)}$ to obtain a set of bootstrap pure term estimates, $\{\hat{v}_S\}^{(b)}$.
    \item These bootstrap moments are plugged into the appropriate closed-form formula described in Section~\ref{sec:rank_identification} to yield the numerator $\text{Num}_p^{(b)}$ and denominator $\text{Den}_p^{(b)}$ for each bootstrapped replicate of the identifiable rank $r_p^{(b)}$:
    \begin{itemize}
        \item PARAFAC/CP (order M, single rank r): for $p \neq q$, use \autoref{eq:parafac:closed}.
        \item TT (order M, ranks $r_1, \dots, r_{M-1}$): for interior ranks $1 < p < M-1$, use \autoref{eq:tt:closed}.
    \end{itemize}
    \item Each bootstrap estimate is regularized to prevent numerical instability.
    For each replicate $b$, we propose a regularized rank
    $$
    \hat{r}_p^{(b)} = \mathrm{sign}(\text{Den}_p^{(b)}) \cdot \frac{\text{Num}_p^{(b)}}{|\text{Den}_p^{(b)}| + \varepsilon_p} .
    $$
    The shrinkage term $\varepsilon_p = 1.96 \cdot \mathrm{SE}(\{\text{Den}_p^{(j)}\}_{j=1}^B)$ is the data-driven standard error of the denominator, estimated from all bootstrap replicates.
\end{enumerate}

\paragraph{The Final Rank Estimate.}
We provide point rank estimates with bootstrapped confidence intervals based on the collection of regularized estimates $\{\hat{r}_p^{(b)}\}_{b=1}^B$:
\begin{itemize}[noitemsep, topsep=0pt, left=4pt]
    \item \textbf{Point Estimate:} The median is used for its robustness to outliers, $\hat{r}_p^{\text{med}} = \text{median}\{ \hat{r}_p^{(b)} \}$.
    \item \textbf{Confidence Interval:} The $(1-\alpha)$ percentile interval is computed from the bootstrap distribution:
    $[\hat{r}_p^{\text{low}}, \hat{r}_p^{\text{up}}] = \text{percentile}(\{ \hat{r}_p^{(b)} \}, [\alpha/2, 1-\alpha/2])$.
\end{itemize}

\section{Empirical Validation}
\label{sec:empirical_validation}

We validate the theoretical findings for the identifiable models using simulation studies, confirming that the closed-form estimators accurately recover the true latent ranks when applied to finite tensor data.

We present results for the PARAFAC/CP model in Section~\ref{ssec:sim_parafac}, and the empirical studies for TT and TR in Appendix~\ref{app:sec:further}.
In addition, we present in Section~\ref{ssec:real_PINCAT} a comparison of the proposed rank estimation procedure and Bayesian inference alternatives on
the real-world PINCAT MRI dataset \citep{candes2013unbiased}\footnote{Link to the dataset \url{https://candes.su.domains/software/SURE/data.html}}.

The codebase with the tensor rank estimation pipeline to replicate the experiments is available in \href{https://github.com/iurteagalab/tensor-ranks}{https://github.com/iurteagalab/tensor-ranks}.

\subsection{Simulations of the PARAFAC/CP model}
\label{ssec:sim_parafac}

\paragraph{Generative Process.}
We generated synthetic tensor-data according to the PARAFAC/CP model, where elements of the latent factors ($\theta^{(p)}$) were drawn independently from Gamma distributions. The shape and scale parameters of the Gamma priors were chosen to control the factor-specific means and variances ($\mu_p, \sigma_p^2$) while ensuring the resulting rate tensor $\etatensor$ was strictly positive. The final observed tensor elements $Y_{\mathbf{i}}$ were then drawn from a Poisson distribution with the corresponding rate $\eta_{\mathbf{i}}$ and $ Y_{\mathbf{i}} \sim \text{Poisson}(\eta_{\mathbf{i}}) \;.$

\paragraph{Experimental Design.}
We study the precision of the constructed pipeline by generating PARAFAC/CP model data in a wide range of conditions:
we vary the tensor order ($M \in \{3,4\}$),
the tensor dimensions ($N_p \in \{ (100, 100, 100), (50, 50, 50, 50) \}$),
and the ground-truth latent ranks ($r_p \in \{ 5, \ldots, 65 \}$).

For each case, we aggregate results over multiple independent runs,
\ie realizations of both the data and the latents,
using a set of values for the prior described in Appendix~\ref{app:sec:further}.
For each run, a median rank was estimated from a distribution of bootstrap replicates, and the plots show the mean of these median estimates.



\begin{figure}[t]
    \centering

    \framebox[0.5\textwidth][c]{  \centering  \includegraphics[width=0.45\textwidth]{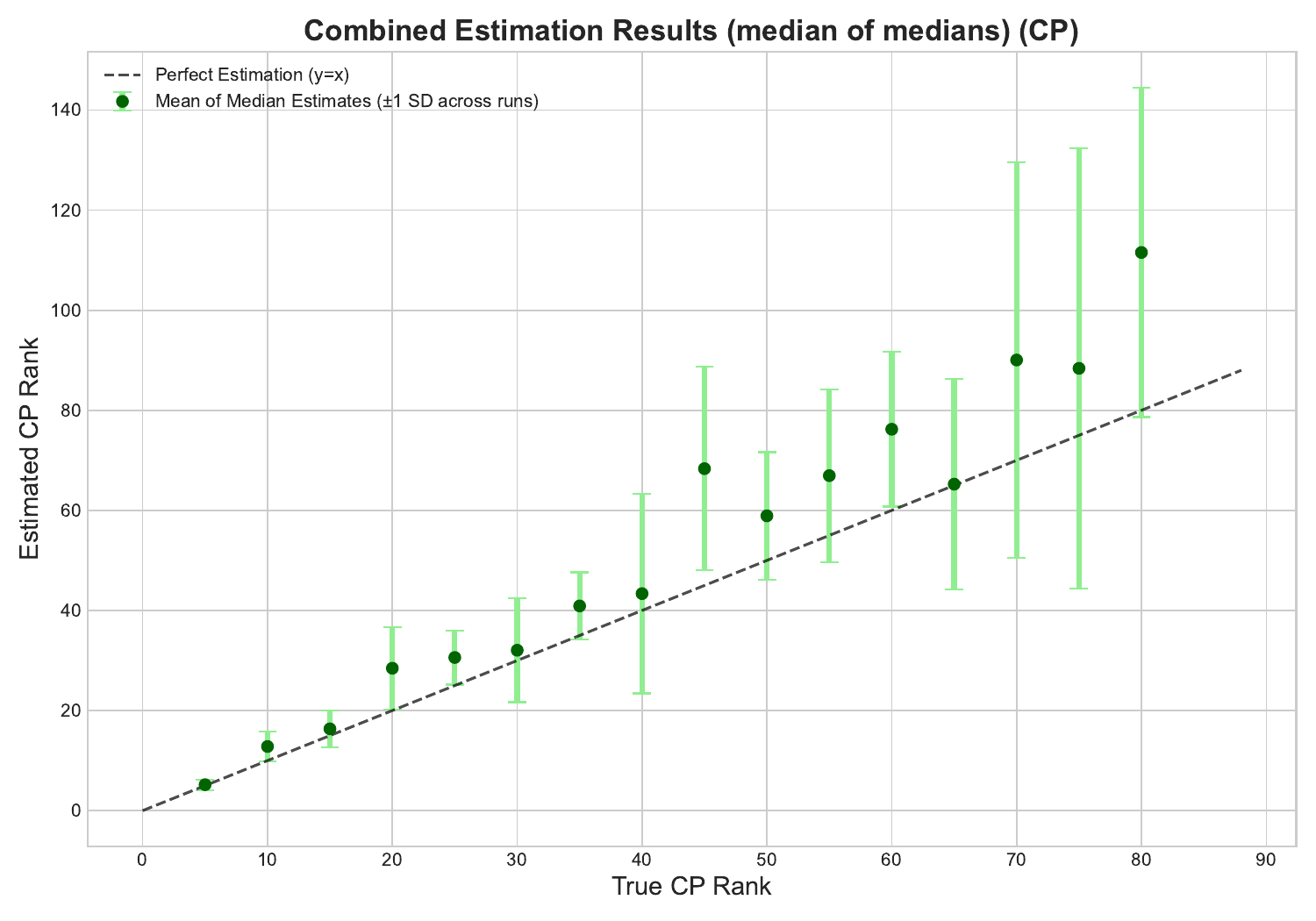}}
    \caption{Combined estimation results for the PARAFAC/CP model. The dashed line represents perfect estimation ($y=x$). Green dots show the mean of median estimates across multiple runs, with error bars indicating $\pm1$ standard deviation.}
    \label{fig:cp_results}
    \vspace{-3ex}
\end{figure}
\paragraph{Rank estimation results.} We provide strong empirical support of the theoretical identifiability and estimation accuracy of the PARAFAC/CP model's ranks.
Figure~\ref{fig:cp_results} shows the estimated CP rank versus the true CP rank, aggregated across all simulation runs.
The mean of the estimated ranks (green dots) closely tracks the line of perfect estimation. For true ranks ranging from low values up to 90, the estimator correctly identifies the latent dimension and follows the linear trend.

Best performance (both in closer mean and reduced estimation uncertainty) is achieved for reduced ($<25$) rank values.
We observe a slight positive bias and an increase in estimation variance (error bars) as the true rank grows. This is expected, as higher-rank structures are more complex, and their corresponding moment-based signatures are more sensitive to the statistical noise inherent in finite data sampling. Overall, the results confirm that the estimator is effective and that the CP rank is practically identifiable.

\paragraph{Comparison with AIC/BIC Rank Estimation.}
To benchmark the efficacy of the proposed prior-predictive (PP) estimator,
we evaluate it alongside information-criteria approaches (AIC/BIC) to rank estimation. 
Notably, the variational CP model coupled with AIC/BIC scoring rules consistently collapsed to the lowest available rank (often $R=2$) in the grid search of our experiments,
across the simulated CP datasets under various prior assumptions (Gamma, Uniform, Gaussian). 

In contrast, the PP estimator consistently recovered the scale of the ground-truth rank across non-zero mean scenarios.\footnote{
If all factors have zero-mean priors,
all pure interaction terms vanish and hence, tensor ranks are not identifiable from second-order moments.
}

\subsection{Real-world tensor-data experiment.}
\label{ssec:real_PINCAT}
We additionally assess rank estimation performance on the real-world PINCAT MRI dataset,
by comparing our automatically determined rank ($\hat{r}_{\text{PP}} = 70$) with the rank values obtained by explicit Bayesian inference methods described by \citet{10.3389/frai.2021.668353}.
Specifically, we compare our approach to the ranks obtained via BFCP and Hamiltonian Monte Carlo (HMC) methods
using a hold-out of 10\% of the data to measure the predictive performance.

Table~\ref{tab:pincat_mri} illustrates that our computationally inexpensive PP estimator matches the performance and uncertainty calibration of full Bayesian procedures.

\begin{table}[h]
\centering
\caption{PINCAT MRI rank estimation and predictive performance against Bayesian baselines.}
\label{tab:pincat_mri}
\resizebox{\columnwidth}{!}{
\begin{tabular}{lccc}
\toprule
Method & Rank & Test RMSE & Test SD \\
\midrule
\textbf{Prior Predictive (ours)} & 70 & 0.6086 & 0.4876 \\
BFCP \citep{10.3389/frai.2021.668353} & 17 & 0.6087 & 0.4877 \\
HMC \citep{10.3389/frai.2021.668353} & 65 & 0.6086 & 0.4876 \\
\bottomrule
\end{tabular}
}
\end{table}

\paragraph{Robustness to Missing Data.}
A crucial aspect of real-world tensor analysis is the presence of missing entries.
We stress-tested our PP estimator under Missing Completely At Random (MCAR) conditions,
applying random masks incrementally (up to 50\% data missingness) on the PINCAT dataset. 

As the percentage of observed data decreases, the variance of the estimator increases (naturally matching the reduction in valid paired observations, $N_{\text{cov}}$).
We observe in the results presented in Table~\ref{tab:missing_mri} that the estimated rank shrinks gradually due to the reduced effective information:
the PP estimator degrades smoothly and predictably, tracing the theoretical finite-sample boundaries.

\begin{table}[h]
\centering
\caption{PINCAT MRI predictive and rank estimation performance under varying MCAR missingness.}
\label{tab:missing_mri}
\begin{tabular}{ccc}
\toprule
Observed \% & Test RMSE & Estimated Rank ($\hat{r}_{\text{PP}}$) \\
\midrule
90\% & 0.6086 & 106.9 \\
80\% & 0.6414 & 83.0 \\
70\% & 0.6781 & 66.4 \\
50\% & 0.7609 & 60.9 \\
\bottomrule
\end{tabular}
\end{table}





\Section{Discussion}

We present a prior predictive moment matching analysis framework revealing that rank identifiability in probabilistic tensor models hinges on the latent factors' interaction topology.

More precisely, we demonstrate that linear (additive) structures like the PARAFAC/CP model yield monomial moments with clean algebraic cancellations, enabling closed-form estimators. Chain-like topologies (TT) introduce asymmetry via contiguous/non-contiguous interactions, providing independent equations per rank that also enable their empirical estimation. The cyclic structure in TR maintains solvability through invertible circulant matrices, despite its symmetry.

In contrast, multiplicative ``hub-and-spoke'' topologies (\eg Tucker) induce linear dependencies, rendering the reduced system degenerate ---we speculate that higher moments or correlation structure in the core tensor could render the models' ranks identifiable. 

The presented topology-based lens offers a principled tensor model analysis toolbox: for identifiable models, low-order moments suffice for exact rank recovery; for others, heuristics like matricization (see Appendix~\ref{app:tucker}) implicitly reduce to CP-like problems. 

Limitations of the presented work include the i.i.d. prior assumption
---correlated prior-based moments remain log-linearizable, but they require an augmented parameterization, analogous to matrix results by \citep{zhu2025prior}---
and the focus on second-order moments only
---we hypothesize extensions to higher-order moments are possible and would broaden the applicability of the presented analysis.

Empirically, our pipeline is shown to be robust to moderate noise, but sensitive to extreme sparsity and data-missingness,
which requires both further numerical considerations and robustness mechanisms.
A natural next-step is the application of the proposed framework to other real-data sets (e.g., on fMRI tensors).

More ambitious future work might consider integration of the prior predictive moment matching approach into variational inference for scalable tensor ML, or exploring non-Euclidean topologies (e.g., hyperbolic embeddings).

A natural generalization of our framework is to consider generic tensor factorizations defined by arbitrary tensor networks~\citep{tensornet_ftml}. In this case, the latent ranks would correspond to the bond dimensions along network edges, and identifiability would depend on the network's topology (e.g., tree-like vs. loopy structures).

\subsubsection*{Acknowledgements}
We would like to thank Caterina de Bacco, Aaron Schein, and Tomasz Ku\'{s}mierczyk for their insightful dialog and constructive feedback provided during the early stages of this research.

E. da Silva. and I. Urteaga were supported by ``la Caixa'' foundation’s LCF/BQ/PI22/11910028 award,
as well as the  Basque Government through the BERC 2022-2025 program
and the Ministry of Science and Innovation's BCAM Severo Ochoa accreditation
CEX2021-001142-S/MICIN/AEI/10.13039/501100011033.

D. Mesquita acknowledges the support by the Fundação Carlos Chagas Filho de Amparo à Pesquisa do Estado do Rio de Janeiro (FAPERJ) (SEI-260003/020348/2025, SEI-260003/020694/2025) and the Conselho Nacional de Desenvolvimento Científico e Tecnológico (CNPq) (404336/2023-0, 305692/2025-9). 

E. da Silva's work is funded by Portuguese national funds through FCT – Foundation for Science and Technology, I.P., within the scope of the research unit UID/00326 - Centre for Informatics and Systems of the University of Coimbra.

I. Urteaga also acknowledges the support of Grant RYC2023-045922-I funded by MICIU/AEI/10.13039/501100011033 and by ESF+. 

\bibliography{biblio}
\bibliographystyle{abbrvnat}

\newpage
\clearpage
\section*{Checklist}



\begin{enumerate}

  \item For all models and algorithms presented, check if you include:
  \begin{enumerate}
    \item A clear description of the mathematical setting, assumptions, algorithm, and/or model. [Yes]
    \item An analysis of the properties and complexity (time, space, sample size) of any algorithm. [Yes]
    \item (Optional) Anonymized source code, with specification of all dependencies, including external libraries. [Yes]
  \end{enumerate}

  \item For any theoretical claim, check if you include:
  \begin{enumerate}
    \item Statements of the full set of assumptions of all theoretical results. [Yes]
    \item Complete proofs of all theoretical results. [Yes] We include some proofs in the main text and other in the appendix.
    \item Clear explanations of any assumptions. [Yes]     
  \end{enumerate}

  \item For all figures and tables that present empirical results, check if you include:
  \begin{enumerate}
    \item The code, data, and instructions needed to reproduce the main experimental results (either in the supplemental material or as a URL). [Yes]
    \item All the training details (e.g., data splits, hyperparameters, how they were chosen). [Yes]
    \item A clear definition of the specific measure or statistics and error bars (e.g., with respect to the random seed after running experiments multiple times). [Yes]
    \item A description of the computing infrastructure used. (e.g., type of GPUs, internal cluster, or cloud provider). [Yes]
  \end{enumerate}

  \item If you are using existing assets (e.g., code, data, models) or curating/releasing new assets, check if you include:
  \begin{enumerate}
    \item Citations of the creator If your work uses existing assets. [Not Applicable]
    \item The license information of the assets, if applicable. [Not Applicable]
    \item New assets either in the supplemental material or as a URL, if applicable. [Not Applicable]
    \item Information about consent from data providers/curators. [Not Applicable]
    \item Discussion of sensible content if applicable, e.g., personally identifiable information or offensive content. [Not Applicable]
  \end{enumerate}

  \item If you used crowdsourcing or conducted research with human subjects, check if you include:
  \begin{enumerate}
    \item The full text of instructions given to participants and screenshots. [Not Applicable]
    \item Descriptions of potential participant risks, with links to Institutional Review Board (IRB) approvals if applicable. [Not Applicable]
    \item The estimated hourly wage paid to participants and the total amount spent on participant compensation. [Not Applicable]
  \end{enumerate}

\end{enumerate}

\clearpage
\appendix
\thispagestyle{empty}

\onecolumn
\aistatstitle{On the Identifiability of Tensor Ranks via Prior Predictive Matching: \\
Supplementary Materials}


\Section{Theoretical Details and Proofs} \label{app:proofs:theory}

This appendix provides complete, self-contained derivations that justify all claims used in the main text while preserving the main text's numbering and statements. We (i) clarify the definitions of observable covariances and pure interaction terms; (ii) give a general counting lemma that turns index-sharing patterns into monomial exponents; and (iii) supply full proofs for Tucker, TT, and TR.

\subsection{Preliminaries: observation model, covariances, and pure terms}
\label{app:prelim}

\paragraph{Notation.}
Let $\Ytensor \in \mathbb{R}^{N_1 \times \dots \times N_M}$ be an observed tensor with $M$ modes. 
A \emph{mode} of an order-$M$ tensor is one of its $M$ dimensions. Let $\mathbf{i} = (i_m)_{m \in [M]} \in \bigtimes_{m=1}^M[N_m]$ be the multi-index from the observed tensor domain, where $[N_m] = \{1, \dots, N_m\}$ . Let $ \mathbf{r}=( r_p )_{p \in [R]} $ be the multi-rank vector of size $R \leq M$. Let the \emph{complete} multi-index be $(i_1, \dots, i_R) \in \bigtimes_{p=1}^{R}  [r_p] $ for the product of \emph{all} $R$ latent rank domains $[r_p]$, and let $\mathbf{i}_{\beta}$ be a \emph{derived} multi-index for latent domains $\beta \in \mathcal{K} \subset [R]$ representing a subset $\mathcal{K}$ of all possible latent dimensions.
For multi-indices $\mathbf{i}, \mathbf{j}$, the shared-mode set is $S(\mathbf{i}, \mathbf{j}) = \{k : i_k = j_k\}$. We use bold sans-serif for tensors ($\etatensor$) and vectors/matrices ($\boldsymbol{\theta}^{(p)}$).
All latent factor multi-indices $\mathbf{i}_p(\mathbf{i},\beta)$ are formed by combining some components of the observed tensor indices $\mathbf{i}$ and some components latent domains $\mathbf{i_\beta}$, and each factor is probabilistically defined as $\theta^{(p)}_{\mathbf{i}_p(\mathbf{i}, \beta)} \sim \pi_p(\mu_p, \sigma_p^2)$,
i.i.d. across indices with $p \in \mathcal{P}$ groups the factors/cores.

\paragraph{Observation model.}
Throughout, conditional on the rate tensor $\etatensor$, entries are conditionally independent:
\[
\E[Y_{\ii}\mid \etatensor]=\eta_{\ii},\qquad 
\Var(Y_{\ii}\mid \etatensor)=\phi(\eta_{\ii}),
\]
for some nonnegative function $\phi(\cdot)$. Hence, for \emph{distinct} entries the \emph{observable cross-covariance} equals the covariance of rates:
\[
\Cov(Y_{\ii},Y_{\jj})=\Cov(\eta_{\ii},\eta_{\jj})\quad \text{if } \ii\neq \jj.
\]
Therefore, all results below are agnostic to the choice of $\phi$ for cross terms and only use $\E[Y]=\E[\eta]$.

\paragraph{Exact-sharing covariances.}
For two multi-indices $\ii,\jj\in \prod_{m=1}^M [N_m]$ and a subset $S\subseteq [M]$, we say the pair \emph{shares exactly $S$} if $i_k=j_k$ for $k\in S$ and $(i_k,j_k)$ are free (not constrained to be equal) for $k\notin S$. The \emph{total observable covariance} for pattern $S$ is
\[
C_S \;:=\; \E\!\left[ \Cov\bigl(Y_{\ii},Y_{\jj}\,\big|\, \text{$\ii,\jj$ share exactly $S$}\bigr) \right],
\]
which depends only on $S$ under the iid factor assumptions.

\paragraph{Pure interaction terms and inclusion–exclusion.}
Define $v_S$ by the Moebius inversion (pure term) over the lattice of subsets:
\begin{align}
C_S \;=\; \sum_{\emptyset\neq S^\prime \subseteq S} v_{S^\prime},
\qquad
v_S \;=\; \sum_{S^\prime\subseteq S} (-1)^{|S|-|S^\prime|}\, {C}_{S^\prime}.
\label{eq:app-pure-def}
\end{align}
In all derivations we work with \emph{exact-sharing} $C_S$ and pure terms $v_S$ defined by \eqref{eq:app-pure-def}.

\paragraph{Augmented interaction sets for Tucker.}
For the Tucker model we sometimes want to distinguish contributions that involve the core tensor $G$ versus factor variances. For clarity, we allow \emph{augmented} symbols such as $v_{\{G\}}$ and $v_{\{G,p\}}$ to denote the pure contribution that is proportional to $\sigma_G^2$ (respectively to $\sigma_G^2\sigma_p^2$) when \emph{all physical modes are shared}. This does not change the observable definition in \eqref{eq:app-pure-def}; it is a bookkeeping device that uniquely decomposes $C_{[M]}$ by variance sources.

\subsubsection{A counting lemma for rank exponents}
\label{app:counting-lemma}

All models considered can be written as
\[
\eta_{\ii} \;=\; \sum_{\beta\in \mathcal{K}} \; \prod_{p\in \mathcal{P}} \theta^{(p)}_{\mathrm{i}_p(\ii,\beta)},
\]
where $\beta$ collects all \emph{latent} indices (their cardinalities are products of the ranks $\{r\}$), and the product ranges over parameter groups $p$ (factors/cores), each drawn iid with $\E[\theta^{(p)}]=\mu_p$ and $\Var(\theta^{(p)})=\sigma_p^2$.

\begin{lemma}[Exponent counting]\label{lem:counting}
Fix a model topology (CP/TT/TR/Tucker), a sharing pattern $S\subseteq [M]$, and consider the second moment
$\E[\eta_{\ii}\eta_{\jj}]$ with $(\ii,\jj)$ constrained to share exactly $S$. Write it as a double sum over two latent index tuples $(\beta,\beta')\in \mathcal{K}\times \mathcal{K}$. Then:
\begin{enumerate}[left=6mm, topsep=2pt,itemsep=2pt]
\item For each parameter group $p$, the factor $\theta^{(p)}$ contributes either a mean-square term $\mu_p^2$ (if the two indices of $\theta^{(p)}$ are independent across $(\beta,\beta')$ or across $(\ii,\jj)$) or a variance term $\sigma_p^2$ (if they coincide).
\item Each latent \emph{link} (summation index) contributes an exponent of $r^\alpha$ where $\alpha\in\{0,1,2\}$ equals the number of \emph{independent} surviving sums across $(\beta,\beta')$ for that link after enforcing the constraints from $S$ and from which factors are in variance mode.
\end{enumerate}
Consequently, every pure term $v_S$ is a \emph{monomial} in $\{r\}\cup\{\mu_p,\sigma_p^2\}$ for CP/TT/TR; for Tucker, $v_{\{G\}}$, $v_{\{p\}}$, and $v_{\{G,p\}}$ are monomials when all physical modes are shared.
\end{lemma}

\begin{proof}
Expand $\E[\eta_{\ii}\eta_{\jj}]$ into a double sum over latent indices for the two entries. Independence across different latent tuples makes the expectation factor across parameter groups $p$. Whether the same random variable appears twice (giving $\sigma_p^2$) or two independent copies (giving $\mu_p^2$) is determined by index equalities induced by $S$ and by whether the two latent index tuples coincide on the arguments of $\theta^{(p)}$. The number of independent sums that remain for each latent link determines the power of the corresponding rank. Subtracting lower-order $C_{S'}$ by inclusion–exclusion isolates the pure term $v_S$, which retains a single monomial. \qedhere
\end{proof}

\subsection{Theoretical analysis of the Tucker model}
\label{app:tucker} 

We work with the standard order-$M$ Tucker rate
\[
\eta_{i_1\ldots i_M} \;=\; \sum_{k_1=1}^{r_1}\cdots \sum_{k_M=1}^{r_M}
G_{k_1\ldots k_M} \,\prod_{p=1}^M \theta^{(p)}_{i_p, k_p},
\]
with iid $G_{k}$ having mean $\mu_G$ and variance $\sigma_G^2$, and iid factor entries $\theta^{(p)}$ having moments $(\mu_p,\sigma_p^2)$.

\paragraph{Global mean.}
By independence, 
\[
\E[\eta_{i_1\ldots i_M}]=\sum_{k_1,\ldots,k_M}\mu_G \prod_{p=1}^{M}\mu_p
= \Bigl(\prod_{p=1}^M r_p\Bigr)\,\mu_G \prod_{p=1}^M \mu_p.
\]
Thus,
\begin{equation}
(E[Y])^2 \;=\; \Bigl(\prod_{p=1}^M r_p\Bigr)^2\, \mu_G^2 \prod_{p=1}^M \mu_p^2.
\label{eq:tucker-mean-sq}
\end{equation}

\paragraph{Pure terms when all modes are shared.}
Consider $C_{[M]}=\Cov(Y_{\ii},Y_{\jj})$ for pairs sharing all $M$ physical indices (this is the only case in which core-variance and factor-variance mix coherently in Tucker). Decomposing by the variance sources we obtain four monomials:

\begin{lemma}[Tucker monomials for shared-all-modes]\label{lem:tucker-monomials-appendix}
With the above conventions, the following hold:
\begin{align}
v_{\{G\}} \;&=\; \sigma_G^2 \prod_{p=1}^M \bigl(r_p^2 \mu_p^2\bigr),
\label{eq:tuck-vG}\\
v_{\{p\}} \;&=\; \mu_G^2 \,\bigl(r_p \sigma_p^2\bigr)\,\prod_{q\neq p} \bigl(r_q^2 \mu_q^2\bigr), 
\label{eq:tuck-vp}\\
v_{\{G,p\}} \;&=\; \sigma_G^2 \,\bigl(r_p \sigma_p^2\bigr)\,\prod_{q\neq p} \bigl(r_q^2 \mu_q^2\bigr),
\label{eq:tuck-vGp}
\end{align}
and \eqref{eq:tucker-mean-sq} for $(E[Y])^2$ remains unchanged.
\end{lemma}

\begin{proof}
Expand $\E[\eta_{\ii}\eta_{\jj}]$ as a double sum over the core indices $(k,k')$ and factor indices for two copies. For the pure term $v_{\{G\}}$, variance arises from $G$ with $k=k'$ (contributing $\sigma_G^2$). All factor groups are in mean mode and appear in \emph{both} copies, leaving two independent summations per factor index; each contributes $r_q^2 \mu_q^2$. Thus $v_{\{G\}}=\sigma_G^2 \prod_q (r_q^2 \mu_q^2)$. 

For $v_{\{p\}}$, variance arises from the single factor at mode $p$ (giving $r_p \sigma_p^2$), while the core and all other factors contribute means in both copies: hence $\mu_G^2$ and, for $q\neq p$, $r_q^2 \mu_q^2$. This yields \eqref{eq:tuck-vp}. 

For $v_{\{G,p\}}$, combine the two variance sources: $G$ in variance and factor $p$ in variance; the remaining factors contribute mean-squared terms across the two copies, producing \eqref{eq:tuck-vGp}. 
\end{proof}

\paragraph{A Tucker identity and its consequence.}
Using \eqref{eq:tucker-mean-sq}–\eqref{eq:tuck-vGp}, we obtain
\begin{align}
\log v_{\{p\}} - 2\log (E[Y]) 
&= \log\!\left[\mu_G^2 (r_p \sigma_p^2)\prod_{q\neq p}\!\bigl(r_q^2 \mu_q^2\bigr)\right] 
 - 2\log\!\left[\Bigl(\prod_{q} r_q\Bigr)\mu_G \prod_q \mu_q\right]\nonumber\\
&= \log\!\left(\frac{\sigma_p^2}{r_p \mu_p^2}\right),
\label{eq:tucker-identity-left}
\end{align}
and similarly
\begin{align}
\log v_{\{G,p\}} - \log v_{\{G\}}
&= \log\!\left[\sigma_G^2 (r_p \sigma_p^2)\prod_{q\neq p}\!\bigl(r_q^2 \mu_q^2\bigr)\right]
- \log\!\left[\sigma_G^2 \prod_q \bigl(r_q^2 \mu_q^2\bigr)\right]\nonumber\\
&= \log\!\left(\frac{\sigma_p^2}{r_p \mu_p^2}\right).
\label{eq:tucker-identity-right}
\end{align}
Hence we have the exact identity
\begin{equation}
\boxed{\quad
\log v_{\{p\}} - 2\log (E[Y]) \;=\; \log v_{\{G,p\}} - \log v_{\{G\}}
\quad}
\label{eq:tucker-identity}
\end{equation}
which is the clarified version of Eq.~(3.1) in the main text (Lemma~\ref{lemma:tucker_moments}/Lemma~\ref{lem:tucker_identity} there).

\begin{theorem}[Non-identifiability of Tucker ranks: full proof]\label{thm:tucker-appendix}
Let $\mathbf{A}\mathbf{x}=\mathbf{b}$ be the log-linear system built from the monomials for $(E[Y])^2$, $v_{\{G\}}$, $v_{\{p\}}$, $v_{\{G,p\}}$ with $p=1,\dots,M$. Then the $M$ columns corresponding to $\log r_p$ lie in the span of the columns that do not involve $\log r_p$; equivalently, the reduced design matrix for $\mathbf{x}_r=(\log r_1,\dots,\log r_M)$ is the zero matrix. Thus the ranks $(r_1,\dots,r_M)$ are not identifiable from second moments.
\end{theorem}

\begin{proof}
Each identity \eqref{eq:tucker-identity} eliminates $\log r_p$ when forming the linear combination of rows
\[
(\log v_{\{p\}}) - 2\log (E[Y]) - \bigl[(\log v_{\{G,p\}}) - \log v_{\{G\}}\bigr] = 0.
\]
Hence the coordinate vector that picks $\log r_p$ is a combination of columns that correspond only to nuisance parameters (log-means and log-variances). Doing this for each $p$ shows the reduced matrix has zero columns. \qedhere
\end{proof}

\subsection{Theoretical analysis of the Tensor Train model}\label{app:tensor_train}  

This appendix provides complete, self-contained derivations for the TT model. We (i) formalize the monomial structure of pure terms $v_S$ for arbitrary exact-sharing patterns $S$, (ii) derive the two key identities used to identify \emph{interior} ranks, and (iii) provide boundary-mode analogues and closed-form estimators for \emph{boundary} ranks.

\paragraph{Model.}
For an order-$M$ TT, the rate is
\[
\eta_{i_1\ldots i_M} \;=\; \sum_{k_1=1}^{r_1}\cdots \sum_{k_{M-1}=1}^{r_{M-1}}
\theta^{(1)}_{1,i_1,k_1}\, \theta^{(2)}_{k_1,i_2,k_2}\cdots \theta^{(M)}_{k_{M-1},i_M,1},
\]
where entries of core $p$ have moments $(\mu_p,\sigma_p^2)$ and are iid across indices. Exact-sharing covariances $C_S$ and pure terms $v_S$ follow Appendix~\ref{app:prelim}.

\subsubsection{A general monomial formula for TT pure terms}

\begin{proposition}[TT pure terms are monomials with explicit rank exponents]\label{prop:TT-monomial}
Let $S\subseteq\{1,\dots,M\}$ be an exact-sharing pattern. For each TT link $\ell\in\{1,\dots,M-1\}$, define
\[
\alpha_\ell(S)\;:=\; 2-\mathbf{1}\{\ell\in S \;\text{or}\; \ell+1\in S\}\;\in\{1,2\}.
\]
Then the pure term associated with $S$ is the \emph{single} monomial
\begin{equation}
\boxed{\qquad
v_S \;=\; \biggl(\prod_{\ell=1}^{M-1} r_\ell^{\alpha_\ell(S)}\biggr)\;
\biggl(\prod_{p\in S} \sigma_p^2\biggr)\;
\biggl(\prod_{p\notin S} \mu_p^2\biggr).
\qquad}
\label{eq:TT-vS-monomial}
\end{equation}
\end{proposition}

\begin{proof}
Expand $\E[\eta_{\ii}\eta_{\jj}]$ as a double sum over latent tuples $(k_1,\dots,k_{M-1})$ and $(k'_1,\dots,k'_{M-1})$. For a fixed $S$, modes in $S$ contribute variance ($\sigma_p^2$), modes not in $S$ contribute squared means ($\mu_p^2$) after subtracting $\E[\eta_{\ii}]\,\E[\eta_{\jj}]$. A TT \emph{link} $\ell$ is the pair $(k_\ell,k'_\ell)$. If neither adjacent physical mode $\ell$ nor $\ell+1$ lies in $S$, the two copies keep independent sums over $k_\ell$ and $k'_\ell$ (yielding $r_\ell^2$). If at least one of $\ell$ or $\ell+1$ lies in $S$, equality constraints collapse two sums into one (yielding $r_\ell$). This gives the exponent $\alpha_\ell(S)$. Moebius inversion over subsets $T\subseteq S$ cancels all mixed patterns, leaving the single monomial \eqref{eq:TT-vS-monomial}.
\end{proof}

\paragraph{Consequences.}
Proposition~\ref{prop:TT-monomial} immediately yields a catalogue of pure terms:
\begin{align*}
\text{Singleton $S=\{p\}$:}\quad
& v_{\{p\}} = \Bigl(\prod_{\ell\ne p-1,p} r_\ell^{2}\Bigr)\, r_{p-1}\, r_{p}\;\; \sigma_p^2 \prod_{k\ne p}\mu_k^2.\\[2mm]
\text{Adjacent pair $S=\{p,p+1\}$:}\quad
& v_{\{p,p+1\}} = \Bigl(\prod_{\ell\notin\{p-1,p,p+1\}} r_\ell^{2}\Bigr)\, r_{p-1}\, r_{p}\, r_{p+1}\;\; \sigma_p^2\sigma_{p+1}^2 \prod_{k\notin\{p,p+1\}}\mu_k^2.\\[2mm]
\text{Noncontiguous pair $S=\{p-1,p+2\}$:}\quad
& v_{\{p-1,p+2\}} = \Bigl(\prod_{\ell\notin\{p-2,p-1,p,p+1,p+2\}} r_\ell^{2}\Bigr)\, r_{p-2}\, r_{p-1}\, r_{p}^{\color{blue}{\mathbf{2}}}\, r_{p+1}\, r_{p+2} \\
&\hspace{2.5cm}\times \;\; \sigma_{p-1}^2\sigma_{p+2}^2 \prod_{k\notin\{p-1,p+2\}}\mu_k^2.\\[2mm]
\text{Triple $S=\{p-1,p,p+2\}$:}\quad
& v_{\{p-1,p,p+2\}} = \Bigl(\prod_{\ell\notin\{p-2,p-1,p,p+1,p+2\}} r_\ell^{2}\Bigr)\, r_{p-2}\, r_{p-1}\, r_{p}^{\color{blue}{\mathbf{1}}}\, r_{p+1}\, r_{p+2} \\
&\hspace{2.5cm}\times \;\; \sigma_{p-1}^2\sigma_{p}^2\sigma_{p+2}^2 \prod_{k\notin\{p-1,p,p+2\}}\mu_k^2.
\end{align*}
The highlighted exponents of $r_p$ differ by one (2 vs.\ 1), which drives the $1/r_p$ factor used below.

\subsubsection*{A.3.2 Interior-mode identities and estimator}

\begin{lemma}[Interior identities]\label{lem:TT-interior-identities}
For any interior mode $p$ with $1<p<M$:
\begin{align}
\frac{v_{\{p,p+1\}}}{v_{\{p+1\}}} &= \frac{\sigma_p^2}{\mu_p^2},
\label{eq:TT-int-id1}\\
\frac{v_{\{p-1,p,p+2\}}}{v_{\{p-1,p+2\}}} &= \frac{1}{r_p}\,\frac{\sigma_p^2}{\mu_p^2}.
\label{eq:TT-int-id2}
\end{align}
\end{lemma}

\begin{proof}
Using \eqref{eq:TT-vS-monomial}: In \eqref{eq:TT-int-id1}, the link exponents coincide for every $\ell$ (both sets include $p+1$), so the rank parts cancel; the moment parts differ only at mode $p$, giving $\sigma_p^2/\mu_p^2$. In \eqref{eq:TT-int-id2}, the moment parts differ only at $p$, again giving $\sigma_p^2/\mu_p^2$, while the rank parts differ by exactly one power of $r_p$ because $\alpha_p(\{p-1,p+2\})=2$ (neither $p$ nor $p+1$ is in the set) but $\alpha_p(\{p-1,p,p+2\})=1$ (now $p\in S$).
\end{proof}

\begin{theorem}[Closed-form estimator for interior TT ranks]\label{thm:TT-interior-estimator}
For $1<p<M-1$,
\begin{equation}
\boxed{\qquad
r_p \;=\; \frac{v_{\{p,p+1\}}\, v_{\{p-1,p+2\}}}{v_{\{p+1\}}\, v_{\{p-1,p,p+2\}}}.
\qquad}
\end{equation}
\end{theorem}
\begin{proof}
Divide \eqref{eq:TT-int-id1} by \eqref{eq:TT-int-id2}.
\end{proof}

\subsubsection{Detailed TT moment calculations (interior and boundary cases)}\label{app:tensor_train_details}

This section expands Proposition~\ref{prop:TT-monomial} for patterns used in practice and covers boundary modes \(p=1\) and \(p=M-1\).

\subsubsubsection{Interior patterns (full monomials)}
For $1<p<M-1$:
\begin{align*}
v_{\{p\}} \;&=\; \Bigl(\!\prod_{\ell\notin\{p-1,p\}} r_\ell^{2}\Bigr)\, r_{p-1}\, r_p \;\;\sigma_p^2 \prod_{k\ne p}\mu_k^2,\\
v_{\{p+1\}} \;&=\; \Bigl(\!\prod_{\ell\notin\{p,p+1\}} r_\ell^{2}\Bigr)\, r_{p}\, r_{p+1} \;\;\sigma_{p+1}^2 \prod_{k\ne p+1}\mu_k^2,\\
v_{\{p,p+1\}} \;&=\; \Bigl(\!\prod_{\ell\notin\{p-1,p,p+1\}} r_\ell^{2}\Bigr)\, r_{p-1}\, r_{p}\, r_{p+1}\;\; \sigma_p^2\sigma_{p+1}^2 \prod_{k\notin\{p,p+1\}}\mu_k^2,\\
v_{\{p-1,p+2\}} \;&=\; \Bigl(\!\prod_{\ell\notin\{p-2,p-1,p,p+1,p+2\}} r_\ell^{2}\Bigr)\, r_{p-2}\, r_{p-1}\, r_{p}^{2}\, r_{p+1}\, r_{p+2}\;\; \sigma_{p-1}^2\sigma_{p+2}^2 \prod_{k\notin\{p-1,p+2\}}\mu_k^2,\\
v_{\{p-1,p,p+2\}} \;&=\; \Bigl(\!\prod_{\ell\notin\{p-2,p-1,p,p+1,p+2\}} r_\ell^{2}\Bigr)\, r_{p-2}\, r_{p-1}\, r_{p}^{1}\, r_{p+1}\, r_{p+2}\;\; \sigma_{p-1}^2\sigma_{p}^2\sigma_{p+2}^2 \prod_{k\notin\{p-1,p,p+2\}}\mu_k^2.
\end{align*}
From these, Lemma~\ref{lem:TT-interior-identities} follows by cancellation.

\subsubsubsection{Boundary patterns and identities}

At the chain ends, the link \(r_1\) sits between modes \(1\) and \(2\); the link \(r_{M-1}\) sits between \(M-1\) and \(M\). We need two identities per boundary rank: one giving \(\sigma_p^2/\mu_p^2\), and a second giving \((\sigma_p^2/\mu_p^2)/r_p\).

\paragraph{Left boundary (\(p=1\)).}
For \(M\ge 3\),
\begin{align}
\frac{v_{\{1,2\}}}{v_{\{2\}}} &= \frac{\sigma_1^2}{\mu_1^2},
\label{eq:TT-left-id1}\\
\frac{v_{\{1,3\}}}{v_{\{3\}}} &= \frac{1}{r_1}\,\frac{\sigma_1^2}{\mu_1^2}.
\label{eq:TT-left-id2}
\end{align}
\emph{Proof.} The first ratio mirrors \eqref{eq:TT-int-id1} with \(p=1\). For the second, note that including mode \(3\) constrains link \(r_2\) in both numerator and denominator, while including mode \(1\) (only in the numerator) additionally constrains \(r_1\), changing its exponent from \(2\) to \(1\). Moments differ only at mode \(1\). Hence the ratio equals \((\sigma_1^2/\mu_1^2)/r_1\).

\paragraph{Right boundary (\(p=M-1\)).}
For \(M \geq 3\),
\begin{align}
\frac{v_{\{M-1,M\}}}{v_{\{M\}}} &= \frac{\sigma_{M-1}^2}{\mu_{M-1}^2},
\label{eq:TT-right-id1}\\
\frac{v_{\{M-3,M-1\}}}{v_{\{M-3\}}} &= \frac{1}{r_{M-1}}\,\frac{\sigma_{M-1}^2}{\mu_{M-1}^2}.
\label{eq:TT-right-id2}
\end{align}
\emph{Proof.} Symmetric to the left boundary with the appropriate shift (now link \(r_{M-1}\) is constrained iff mode \(M-1\) or \(M\) is in \(S\)). The pair \(\{M-3\}\) vs.\ \(\{M-3,M-1\}\) plays the role of \(\{3\}\) vs.\ \(\{1,3\}\).

\begin{theorem}[Closed-form estimators for boundary TT ranks]\label{thm:TT-boundary-estimators}
For \(M\ge 3\), the boundary ranks are identifiable via
\begin{align}
\boxed{\quad
r_1 \;=\; \frac{v_{\{1,2\}}\, v_{\{3\}}}{v_{\{2\}}\, v_{\{1,3\}}}, 
\qquad
r_{M-1} \;=\; \frac{v_{\{M-1,M\}}\, v_{\{M-3\}}}{v_{\{M\}}\, v_{\{M-3,M-1\}}}.
\quad}
\end{align}
\end{theorem}
\begin{proof}
Divide \eqref{eq:TT-left-id1} by \eqref{eq:TT-left-id2} (left boundary) and \eqref{eq:TT-right-id1} by \eqref{eq:TT-right-id2} (right boundary).
\end{proof}

\paragraph{Edge-case notes.}
(i) For \(M=3\), the identities remain valid with the natural interpretation of sets (e.g., \(\{3\}\) and \(\{1,3\}\) are defined).  
(ii) For \(M=2\), TT reduces to a matrix factorization with a single rank and the CP formula applies.

\subsubsubsection{Worked example (rank exponents table)}
For illustration, consider \(M=6\) and \(p=3\). The link exponents \(\alpha_\ell(S)\) from Proposition~\ref{prop:TT-monomial} are:

\begin{center}
\begin{tabular}{l|cccccc}
$S$ & $\alpha_1$ & $\alpha_2$ & $\alpha_3$ & $\alpha_4$ & $\alpha_5$ \\
\hline
$\{3\}$ & $2$ & $1$ & $1$ & $2$ & $2$ \\
$\{4\}$ & $2$ & $2$ & $1$ & $1$ & $2$ \\
$\{3,4\}$ & $2$ & $1$ & $1$ & $1$ & $2$ \\
$\{2,5\}$ & $1$ & $1$ & $\color{blue}{2}$ & $1$ & $1$ \\
$\{2,3,5\}$ & $1$ & $1$ & $\color{blue}{1}$ & $1$ & $1$ \\
\end{tabular}
\end{center}

The highlighted entries show the drop \(2\to 1\) at link \(3\) when moving from \(\{2,5\}\) to \(\{2,3,5\}\), producing the \(1/r_3\) factor.

\medskip
\noindent
This completes the TT analysis with full monomial formulas, interior and boundary identities, and closed-form estimators for all TT ranks.

\subsection{Theoretical analysis of the Tensor Ring model} \label{app:tensor_ring} 

We analyze the Tensor Ring (TR) model in full detail. Despite its cyclic symmetry, the first/second moments determine all ranks via a short linear system.

\subsubsection{Model and notation}

\begin{definition}[Tensor Ring (TR)]
For an order-$M$ tensor with TR ranks $(r_1,\ldots,r_M)$ (indices modulo $M$), the rate is
\[
\eta_{i_1\ldots i_M} \;=\; \mathrm{Tr}\!\Big( \Theta^{(1)}_{i_1}\Theta^{(2)}_{i_2}\cdots \Theta^{(M)}_{i_M}\Big),
\]
where $\Theta^{(p)}_{i_p}\in \Real^{\,r_{p-1}\times r_p}$ has i.i.d.\ entries with moments $(\mu_p,\sigma_p^2)$, and $r_0\equiv r_M$ closes the ring. Conditional on $\etatensor$, the $Y_{\ii}$ are independent with $\E[Y_{\ii}\mid \etatensor]=\eta_{\ii}$.
\end{definition}

Let the TR \emph{latent indices} be $s_0,s_1,\dots,s_{M-1}$ with $s_0=s_M$ and $s_p\in[r_p]$. Expanding the trace gives
\begin{equation}
\eta_{\ii}
=\!\!\sum_{s_0,\ldots,s_{M-1}}
\prod_{p=1}^{M} \theta^{(p)}_{s_{p-1},\,i_p,\,s_p},
\qquad s_0=s_M,
\label{eq:TR-expansion}
\end{equation}
where $\theta^{(p)}$ denotes an entry of $\Theta^{(p)}$. As in Appendix~\ref{app:prelim}, for exact-sharing $S\subseteq[M]$ we work with $C_S=\Cov(\eta_{\ii},\eta_{\jj})$ and the pure terms $v_S$ via inclusion–exclusion.

\subsubsection{TR moments: monomials and rank exponents}

We first obtain the three families of monomials used throughout: $(E[Y])^2$, singletons $v_{\{p\}}$, and adjacent pairs $v_{\{p,p+1\}}$.

\begin{lemma}[TR prior predictive moments]\label{lem:TR-moments}
For any $M\ge 3$ (indices modulo $M$), the following hold:
\begin{align}
(E[Y])^2 &= \Big(\prod_{k=1}^{M} r_k \mu_k\Big)^{\!2}, 
\label{eq:TR-mean-sq}\\[1mm]
v_{\{p\}} &= \Big(\prod_{k\notin\{p-1,p\}} r_k^2\Big)\; r_{p-1}\; r_{p}^{\,2}\;\;
\sigma_p^2 \prod_{k\ne p}\mu_k^2,
\label{eq:TR-singleton}\\[1mm]
v_{\{p,p+1\}} &= \Big(\prod_{k=1}^{M} r_k\Big)\;\;
\sigma_p^2\sigma_{p+1}^2 \prod_{k\notin\{p,p+1\}}\mu_k^{2}.
\label{eq:TR-adjacent}
\end{align}
\end{lemma}

\begin{proof}
\emph{Mean.} From \eqref{eq:TR-expansion}, $\E[\eta_{\ii}]=\sum_{s} \prod_p \mu_p=\prod_p (r_p\mu_p)$, hence \eqref{eq:TR-mean-sq}.

\emph{Singleton $v_{\{p\}}$.} Consider $\Cov(\eta_{\ii},\eta_{\jj})$ for pairs sharing all modes and isolate the pure contribution where only $\Theta^{(p)}$ is in variance (Appendix~\ref{app:prelim}). Expanding $\E[\eta_{\ii}\eta_{\jj}]$ yields a double sum over latent chains $(s_0,\dots,s_{M-1})$ and $(t_0,\dots,t_{M-1})$. Modes $k\ne p$ contribute $\mu_k^2$; mode $p$ contributes $\sigma_p^2$. The rank exponents count \emph{independent} sums per link:
\begin{itemize}
\item Links not adjacent to $p$ (i.e., $k\notin\{p-1,p\}$) are unconstrained in both copies $\Rightarrow$ 2 independent sums $\Rightarrow r_k^2$.
\item The link $p-1$ is adjacent to the sole variance and becomes \emph{constrained across copies} by equality of the index that flows into $\Theta^{(p)}$; only one joint sum remains $\Rightarrow r_{p-1}$.
\item The link $p$ (outgoing from $\Theta^{(p)}$) remains \emph{unconstrained} across copies because the variance identification at $p$ does not couple the outgoing indices in the two copies; this keeps two independent sums $\Rightarrow r_p^2$.
\end{itemize}
This gives \eqref{eq:TR-singleton}.

\emph{Adjacent pair $v_{\{p,p+1\}}$.} Now only $\Theta^{(p)}$ and $\Theta^{(p+1)}$ are in variance, all other modes in mean. In the ring, placing two consecutive variances \emph{glues the two copies into a single cyclic loop}: every link becomes constrained across the copies, leaving exactly one independent sum per link. Hence each $r_k$ appears to the first power, and the moment factor is $\sigma_p^2\sigma_{p+1}^2\prod_{k\notin\{p,p+1\}}\mu_k^2$, yielding \eqref{eq:TR-adjacent}.
\end{proof}

\paragraph{Discussion.}
The asymmetric exponents in \eqref{eq:TR-singleton} (power 1 on $r_{p-1}$, but power 2 on $r_p$) are a genuine ring effect: the variance at node $p$ constrains the \emph{incoming} link (shared across copies) but leaves the \emph{outgoing} link with independent sums. In contrast, two adjacent variances collapse the two-copy ring globally, forcing single powers for \emph{all} links in \eqref{eq:TR-adjacent}.

\subsubsection{Rank–ratio identity and identifiability}

Define the observable statistic
\begin{equation}
\xi_p \;:=\; \frac{(E[Y])^2\; v_{\{p,p+1\}}}{v_{\{p\}}\; v_{\{p+1\}}}.
\label{eq:TR-xi}
\end{equation}

\begin{proposition}[Rank–ratio identity]\label{prop:TR-rank-ratio}
For all $p$ (indices modulo $M$),
\begin{equation}
\boxed{\qquad
\xi_p \;=\; r_p \,\frac{r_{p-1}}{r_{p+1}}.
\qquad}
\label{eq:TR-ratio}
\end{equation}
\end{proposition}

\begin{proof}
Substitute \eqref{eq:TR-mean-sq}–\eqref{eq:TR-adjacent} into \eqref{eq:TR-xi}. All $\mu$ and $\sigma$ factors cancel. For ranks, the numerator contributes $\big(\prod_k r_k\big)\cdot \big(\prod_k r_k\big)$, while the denominator contributes $(\prod_{k\notin\{p-1,p\}} r_k^2) r_{p-1} r_p^2 \cdot (\prod_{k\notin\{p,p+1\}} r_k^2) r_{p} r_{p+1}^2$. After cancellation, only the factor $r_p r_{p-1}/r_{p+1}$ remains.
\end{proof}

Let $x_p=\log r_p$ and $\psi_p=\log\xi_p$. Taking logs in \eqref{eq:TR-ratio} gives the \emph{circulant} system
\begin{equation}
x_p + x_{p-1} - x_{p+1} \;=\; \psi_p,\qquad p=1,\dots,M,
\label{eq:TR-circulant-eq}
\end{equation}
i.e., $\mathbf{C}\mathbf{x}=\boldsymbol{\psi}$ with the $M\times M$ circulant matrix whose first row is $(1,-1,0,\dots,0,1)$:
\begin{equation}\label{eq:tr:circulant}
\mathbf{C}=\begin{bmatrix}
1 & -1 & 0 & \cdots & 0 & 1 \\
1 & 1 & -1 & \cdots & 0 & 0 \\
0 & 1 & 1 & \cdots & 0 & 0 \\
\vdots & & \ddots & \ddots & & \vdots \\
-1 & 0 & 0 & \cdots & 1 & 1
\end{bmatrix}.
\end{equation}

\begin{theorem}[TR identifiability via DFT]\label{thm:TR-ident}
For every integer $M\ge 3$, the matrix $\mathbf{C}$ in \eqref{eq:tr:circulant} is invertible. Hence $\mathbf{x}=\mathbf{C}^{-1}\boldsymbol{\psi}$ is uniquely determined from moments, and $r_p=\exp(x_p)$ are identifiable for all $p$.
\end{theorem}

\begin{proof}
Let $\omega=e^{2\pi i/M}$. The eigenpairs of a circulant are obtained by the DFT: the $k$th eigenvalue is
\[
\lambda_k \;=\; 1 + \omega^k - \omega^{-k}
\;=\; 1 + 2i\,\sin(2\pi k/M),\qquad k=0,\dots,M-1.
\]
For $M\ge 3$, $\sin(2\pi k/M)=0$ iff $k\in\{0,M\}$; in either case $\lambda_k=1\ne 0$. Thus no $\lambda_k$ vanishes and $\det\mathbf{C}\ne 0$. Therefore \eqref{eq:TR-circulant-eq} has a unique solution $\mathbf{x}$. Exponentiating gives the ranks.
\end{proof}

\paragraph{Computation.}
Practically, solve $\mathbf{C}\mathbf{x}=\boldsymbol{\psi}$ by FFTs: compute $\widehat{\psi}=\mathrm{DFT}(\psi)$, divide componentwise by $\lambda_k$, and invert the DFT to obtain $x$.

\subsubsection{Detailed TR moment calculations and edge cases}
\label{app:TR-details}

We briefly expand the counting in Lemma~\ref{lem:TR-moments}.

\paragraph{Singleton \(v_{\{p\}}\).}
Write the two copies of the ring with latent chains $(s_0,\dots,s_{M})$ and $(t_0,\dots,t_M)$ (with $s_0=s_M$, $t_0=t_M$). The variance at $p$ identifies the \emph{same} entry of $\Theta^{(p)}$ across copies, which forces $s_{p-1}=t_{p-1}$ but leaves $s_p$ and $t_p$ independent. Propagating around the cycle:
\begin{itemize}
\item For links far from $p$ (neither endpoint adjacent to $p$), both copies retain independent sums $\Rightarrow r_k^2$.
\item At link $p-1$ the identification $s_{p-1}=t_{p-1}$ leaves one sum $\Rightarrow r_{p-1}$.
\item At link $p$ the two outgoing indices remain independent across copies $\Rightarrow r_p^2$.
\end{itemize}
Moment factors follow since only mode $p$ contributes $\sigma_p^2$; all others contribute $\mu_k^2$. This reproduces \eqref{eq:TR-singleton}.

\paragraph{Adjacent pair \(v_{\{p,p+1\}}\).}
Variance at $p$ and $p+1$ identifies $\Theta^{(p)}$ and $\Theta^{(p+1)}$ across copies. In the ring, this creates a \emph{single} joint cycle: starting from any link, equality constraints propagate all the way around, leaving exactly one independent sum per link $\Rightarrow \prod_k r_k$. The moment factor is $\sigma_p^2\sigma_{p+1}^2\prod_{k\notin\{p,p+1\}}\mu_k^2$, giving \eqref{eq:TR-adjacent}.

\paragraph{Small-$M$ edge cases.}
\begin{itemize}
\item $M=2$: the TR degenerates into a matrix factorization with a single rank parameter; the CP formulas apply.
\item $M=3$: all formulas above hold; $\mathbf{C}$ has eigenvalues $1,\,1\pm i\sqrt{3}$ (all nonzero), so identifiability still holds.
\end{itemize}

\paragraph{Robustness to minor model variants.}
If the entries of $\Theta^{(p)}$ are centered (\(\mu_p=0\)), the above uses $(E[Y])^2$ only via $\prod \mu_k$, which would be zero; in that case one can replace $(E[Y])^2$ with $\Var(\eta)$ and form analogous ratios using higher moments. Our main identifiability result relies on $\mu_p\ne 0$, as assumed in the main text.

\section{Computational and Statistical Properties of the Estimators}\label{app:comp_stat_est}
\subsection{Covariances, interaction sets, and pure terms}\label{app:cov_v_pure}

We estimate ranks in two steps.
\begin{enumerate}
    \item \textbf{Observable covariances.} For a sharing set \(S\subseteq\{1,\dots,M\}\), the exact–sharing covariance \(\widehat C_S\) is computed from pairs \((Y_{\mathbf i},Y_{\mathbf j})\) with indices equal on \(S\) and different on the complement.
    \item \textbf{Pure terms.} The pure interaction \(\widehat v_S\) is obtained from \(\{\widehat C_{S'}:{S'}\subseteq S\}\) via Möbius inversion:
\[
\widehat v_S \;=\; \sum_{{S'}\subseteq S} (-1)^{|S|-|{S'}|}\,\widehat  C_{S'} .
\]
These \(\widehat v_S\) are the inputs to the closed-form estimators in CP/TT/TR.
\end{enumerate}

\paragraph{CP (PARAFAC).}
With i.i.d.\ factor priors, \(v_S = R\prod_{p\in S}\sigma_p^2\prod_{p\notin S}\mu_p^2\) for any nonempty \(S\). Hence
\[
\widehat R_S \;=\; \frac{\widehat v_S}{\prod_{p\in S}\sigma_p^2\prod_{p\notin S}\mu_p^2}, 
\quad\text{aggregate over } S\in\{\{p\}\}\ \text{and optionally } \{p,p{+}1\}.
\]
Pure terms needed:
\[
\widehat v_{\{p\}}=\widehat C_{\{p\}},\qquad 
\widehat v_{\{i,j\}}=\widehat C_{\{i,j\}}-\widehat C_{\{i\}}-\widehat C_{\{j\}}.
\]

\paragraph{Tensor Train (TT).}
Let \(r_1,\ldots,r_{M-1}\) be bond ranks (open chain). Using the monomial formula for pure terms (Appendix~\ref{app:tensor_train}), the closed forms are:
\begin{align*}
\text{Interior } (1<p<M-1):\quad
& r_p \;=\; \frac{v_{\{p,p+1\}}\; v_{\{p-1,p+2\}}}{v_{\{p+1\}}\; v_{\{p-1,p,p+2\}}}.\\
\text{Left boundary } (p=1,\ M\ge3):\quad
& r_1 \;=\; \frac{v_{\{1,2\}}\; v_{\{3\}}}{v_{\{2\}}\; v_{\{1,3\}}}.\\
\text{Right boundary } (p=M-1,\ M\ge4):\quad
& r_{M-1} \;=\; \frac{v_{\{M-1,M\}}\; v_{\{M-3\}}}{v_{\{M\}}\; v_{\{M-3,M-1\}}}.
\end{align*}
Special cases use triplets directly:
\[
\textstyle M{=}3:\ 
r_1=\frac{v_{\{1,2\}}v_{\{3\}}}{v_{\{2\}}v_{\{1,3\}}},\quad
r_2=\frac{v_{\{2,3\}}v_{\{1\}}}{v_{\{3\}}v_{\{1,2\}}};
\qquad
M{=}4:\ 
\begin{cases}
r_1=\dfrac{v_{\{1,2\}}}{v_{\{2\}}}\cdot\dfrac{v_{\{3,4\}}}{v_{\{1,3,4\}}},\\[4pt]
r_2=\dfrac{v_{\{2,3\}}}{v_{\{3\}}}\cdot\dfrac{v_{\{1,4\}}}{v_{\{1,2,4\}}},\\[4pt]
r_3=\dfrac{v_{\{3,4\}}}{v_{\{2\}}}\cdot\dfrac{v_{\{1,2\}}}{v_{\{1,3,4\}}}.
\end{cases}
\]
Required exact–sharing sets are: all singletons; adjacent pairs; the gap-two pair \(\{p{-}1,p{+}2\}\); the triplet \(\{p{-}1,p,p{+}2\}\); and the boundary helpers \(\{1,3\}\), \(\{M{-}3\}\), \(\{M{-}3,M{-}1\}\). For \(M\in\{3,4\}\), include \(\{1,2,3\}\) and \(\{1,3,4\}\), \(\{1,2,4\}\) as needed.

\paragraph{Tensor Ring (TR).}
Let \(\psi_p=\log(E[Y]^2)+\log v_{\{p,p+1\}}-\log v_{\{p\}}-\log v_{\{p+1\}}\) (indices mod \(M\)). Then the circulant system \(C\,\log r=\psi\) recovers all \(r_p\). Only singletons and adjacent pairs are required.

\paragraph{Sampling kernel (all models).}
For each required set \(S\), draw \(S_{\text{cov}}\) pairs per bootstrap with exact sharing pattern \(S\), compute \(\widehat C_S\), and obtain \(\widehat v_S\) by Möbius inversion over the (small) downward-closed family that contains \(S\).

\subsubsection{Complexity Analysis (Simplified).}
Let \(\mathcal S\) be the number of distinct sharing sets used, \(B\) bootstraps, and \(S_{\text{cov}}\) pairs per set. Pair generation and access are \(O(1)\) amortized; Möbius is \(O(1)\) per set since \(|S|\le3\) here. The dominant cost is sampling:
\[
\boxed{\ \text{Time} \;\approx\; \Theta(\mathcal S\, B\, S_{\text{cov}}),\qquad
\text{Memory} \;\approx\; \Theta(\mathcal S\, B)\ \text{(store-all)}\ \text{or}\ \Theta(\mathcal S)\ \text{(stream)}. }
\]
Model-specific \(\mathcal S\) (to estimate all ranks):
\[
\begin{array}{lcl}
\text{CP:} & \mathcal S = M\ \text{(singletons)}\ \text{or}\ 2M\ \text{(+ adjacent pairs)}.\\
\text{TT:} & \mathcal S = C(M{-}1),\ \text{with small constant }C \ (\text{singletons, adjacents, gap-two, triplets}).\\
\text{TR:} & \mathcal S = 2M\ \text{(singletons + adjacent pairs)}\ \text{+ one tiny linear solve }O(M\log M).
\end{array}
\]
Thus all three scale \emph{linearly} in \(M\) via \(\mathcal S\); TR has the cheapest set family, TT adds a constant-factor overhead for gap-two/triplets, and CP is minimal with singletons.

\subsection{Variance and SNR of the Pure-Term Estimators}
\label{app:tt-variance-snr}

Among the family of moment-based estimators derived in this work, the \emph{pure-term estimator}
for the Tensor Train (TT) structure is the only one that relies \emph{exclusively on second-order moments}—
specifically, on covariances between entries sharing particular subsets of tensor modes.
Unlike estimators that exploit higher-order cross-moments or structured mean terms, this construction
avoids the need for explicit fourth-order moment tensors, making it computationally efficient and broadly applicable.
However, this same property also makes it particularly sensitive to the \emph{signal-to-noise ratio (SNR)} of the
empirical covariances, because all rank information is inferred from ratios of covariances that may have small
expected magnitude and high sampling variability.

In low-SNR regimes, random fluctuations in the empirical covariances can dominate the ratio computations,
leading to unstable or biased rank estimates even when the theoretical ratio identities are exact.
This motivates the detailed SNR analysis that follows: we study how the prior mean $\mu$ and variance (via the coefficient of variation $\mathrm{cv}$) of the latent factors jointly determine the effective SNR of the
covariance-based estimator. The goal is to identify prior configurations that maximize the estimator’s robustness without sacrificing the
closed-form structure that makes it analytically tractable.

Let $v_S$ be the TT pure interaction defined by exact–sharing Moebius inversion over covariances $C_S$.
With per-mode prior mean $\mu_m$, variance $\sigma_m^2$, second raw moment $m_{2,m}=\mu_m^2+\sigma_m^2$ and
fourth raw moment $m_{4,m}$, the TT population pure term is
\[
v_S \;=\;\Bigg(\prod_{\ell=1}^{M-1} r_\ell^{\,e_\ell(S)}\Bigg)
\Bigg(\prod_{m\in S}\sigma_m^2\Bigg)\Bigg(\prod_{m\notin S}\mu_m^2\Bigg),
\quad
e_\ell(S)=\mathbf{1}\{\ell\in S\ \text{or}\ \ell+1\in S\}+1.
\]
Let $\widehat C_{S'}$ be the exact–sharing covariance estimated from $n$ paired samples and
$\widehat v_S=\sum_{{S'}\subseteq S}(-1)^{|{S'}|-|{S'}|}\widehat C_{S'}$.
Under independent paired sampling per ${S'}$ and large $n$,
\[
\mathrm{Var}(\widehat C_{S'})\approx \frac{1}{n}
\Bigg(\prod_{\ell} r_\ell^{\,h_\ell({S'})}\Bigg)\Bigg(\prod_{m\in {S'}} m_{4,m}\Bigg)\Bigg(\prod_{m\notin {S'}} m_{2,m}^2\Bigg)
-\frac{1}{n}
\Bigg(\prod_{\ell} r_\ell^{\,2e_\ell({S'})}\Bigg)\Bigg(\prod_{m\in {S'}} m_{2,m}^2\Bigg)\Bigg(\prod_{m\notin {S'}}\mu_m^4\Bigg),
\]
where $h_\ell({S'})=2e_\ell(T{S'}\in\{2,4\}$ are the fourth–moment exponents for TT.
Consequently,
\[
\mathrm{Var}(\widehat v_S)\;\approx\;\sum_{{S'}\subseteq S}\frac{1}{n}(-1)^{2|S|-2|{S'}|}
\left[
\Bigg(\prod_{\ell} r_\ell^{\,h_\ell({S'})}\Bigg)\Bigg(\prod_{m\in {S'}} m_{4,m}\Bigg)\Bigg(\prod_{m\notin {S'}} m_{2,m}^2\Bigg)
-
\Bigg(\prod_{\ell} r_\ell^{\,2e_\ell({S'})}\Bigg)\Bigg(\prod_{m\in {S'}} m_{2,m}^2\Bigg)\Bigg(\prod_{m\notin {S'}}\mu_m^4\Bigg)
\right].
\]

A convenient leading–order SNR follows by keeping the dominant term in the bracket:
\[
\mathrm{SNR}(S)\equiv \frac{v_S}{\sqrt{\mathrm{Var}(\widehat v_S)}}\;\approx\;
\sqrt{n}\;\prod_{m\in S}\frac{\sigma_m^2}{\sqrt{m_{4,m}}}\;\prod_{m\notin S}\frac{\mu_m^2}{m_{2,m}}
\;\times\;
\prod_{\ell} r_\ell^{\,e_\ell(S)-\tfrac12 h_\ell(S)}.
\]
For TT, $h_\ell(S)=2e_\ell(S)$, so the rank factor cancels and
\[
\boxed{\ \mathrm{SNR}(S)\approx \sqrt{n}\;\prod_{m\in S}\frac{\sigma_m^2}{\sqrt{m_{4,m}}}\;\prod_{m\notin S}\frac{\mu_m^2}{\mu_m^2+\sigma_m^2}\ .\ }
\]

\paragraph{Gamma prior (equal across modes).}
If $\theta^{(m)}\sim\mathrm{Gamma}(\alpha,\theta)$ with mean $\mu$ and $\mathrm{cv}^2=1/\alpha$, then
$m_2=\mu^2(1+\mathrm{cv}^2)$ and $m_4/\mu^4=1+6\mathrm{cv}^2+3\mathrm{cv}^4$. Therefore
\[
\mathrm{SNR}(S)\approx
\sqrt{n}\;\mu^{|S|}\left(\frac{\mathrm{cv}^2}{\sqrt{1+6\mathrm{cv}^2+3\mathrm{cv}^4}}\right)^{|S|}
\left(\frac{1}{1+\mathrm{cv}^2}\right)^{M-|S|}.
\]
This exhibits the empirical trade–off observed in practice: increasing $\mathrm{cv}$ boosts the shared–mode factor but diminishes the non–shared factor; an intermediate $\mathrm{cv}$ maximizes SNR for fixed $|S|$ and $M$.

\paragraph{Graphical analysis.} We can observe from \autoref{fig:snr-theory} that the SNR of the pure-interaction covariance estimator for $|S|=1$ and $|S|=2$ (different sizes of interaction sets) has a narrow band leading to joint high SNR. In most areas of the graph we have prior configurations that lead to low SNR, affecting the quality and variability of any estimator based on pure-interactin covariance values.

\begin{figure}[ht!]
  \centering
  \includegraphics[width=.48\linewidth]{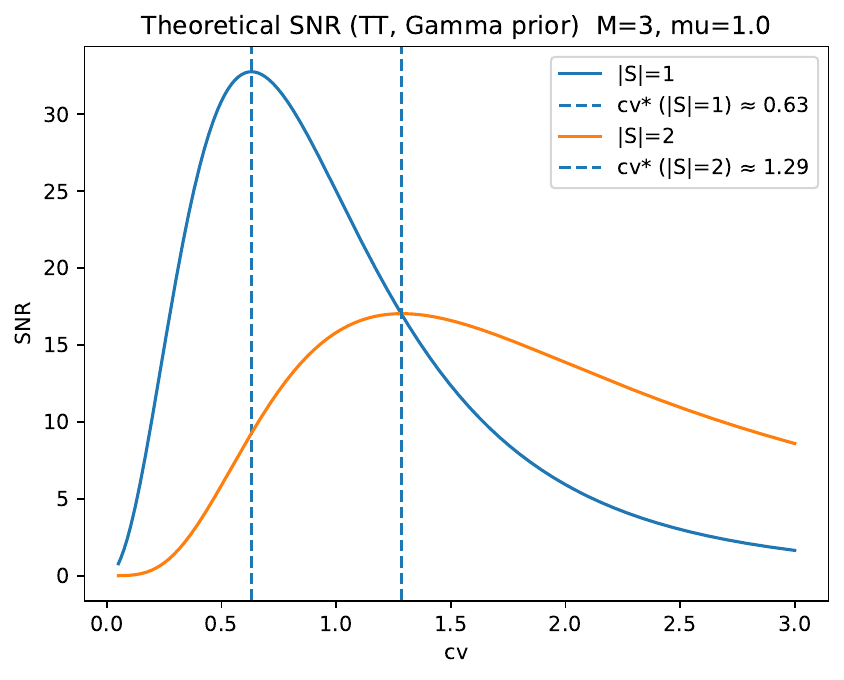}\hfill
  \includegraphics[width=.48\linewidth]{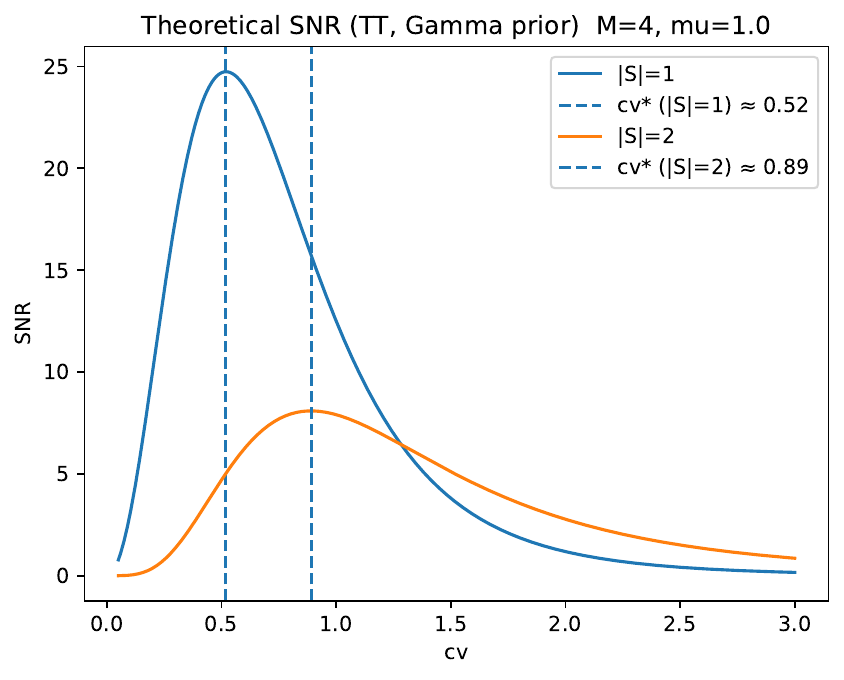}\hfill
  \includegraphics[width=.48\linewidth]{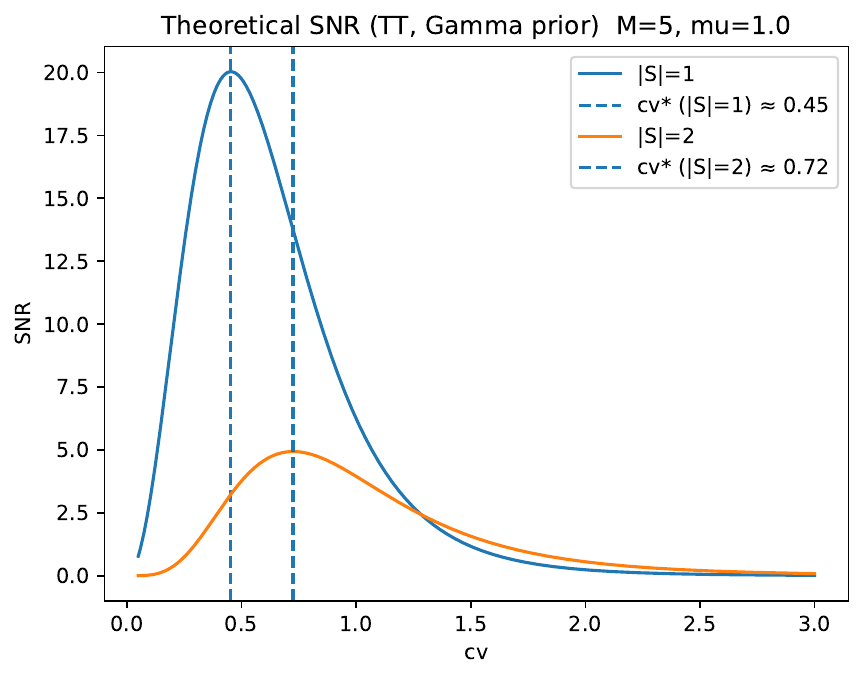}
  \caption{Theoretical SNR (TT, Gamma prior) for $M\in\{3, 4,5\}$, $\mu=1.0$. Vertical lines mark theoretical optimal region for $\mathrm{cv}^{*}$.}
  \label{fig:snr-theory}
\end{figure}

\subsection{Bias Analysis of the Regularized (and Mean-Normalized) Rank Estimators}
\label{app:bias_analysis}

All closed-form rank estimators in this work are \emph{ratios of estimated moments}, and hence inherit the well-known sensitivity of ratio estimators when denominators are small or noisy. We analyze the regularized estimator used throughout and the effect of the optional \emph{mean normalization} (dividing each $\hat v_S$ by $(\widehat{E}[Y])^2$ before forming ratios).

\paragraph{Regularized ratio.}
For any target rank written as $r_p = \mathrm{Num}_p/\mathrm{Den}_p$, we use
\begin{equation}
\hat{r}_p^{\mathrm{reg}}
\;=\;
\mathrm{sign}\!\big(\widehat{\mathrm{Den}}_p\big)\;
\frac{\widehat{\mathrm{Num}}_p}{\left|\widehat{\mathrm{Den}}_p\right| + \varepsilon_p},
\qquad
\varepsilon_p \;=\; 1.96\cdot \mathrm{SE}_{\mathrm{boot}}\!\big(\widehat{\mathrm{Den}}_p\big),
\label{eq:reg-ratio}
\end{equation}
where $\mathrm{SE}_{\mathrm{boot}}$ is the bootstrap standard error computed \emph{across} bootstrap replicates.
The small additive $\varepsilon_p$ stabilizes divisions by noisy denominators. Using $\mathrm{sign}(\cdot)$ preserves the correct orientation when sampling variation flips the estimated sign (rare but possible for small-$S$ pure terms).

\paragraph{Mean normalization.}
Define mean-normalized pure terms $\tilde v_S := \hat v_S/\big((\widehat{E}[Y])^2 + \varepsilon_E\big)$ with a tiny $\varepsilon_E>0$ (e.g.\,$10^{-12}$). All CP/TT/TR rank identities are homogeneous in $v_S$, so replacing $v_S$ by $\tilde v_S$ \emph{does not change the population target}. In finite samples it reduces dispersion because we scale by a high-SNR quantity, $(\widehat{E}[Y])^2$. We therefore analyze both \eqref{eq:reg-ratio} with $(\widehat{\mathrm{Num}},\widehat{\mathrm{Den}})$ built from $\hat v_S$ or from $\tilde v_S$.

\paragraph{Delta-method bias.}
Let $R=\widehat{\mathrm{Num}}/\widehat{\mathrm{Den}}$ and $f(x,y)=x/(y+\varepsilon)$ for a fixed $\varepsilon\ge 0$.
A first-order delta-method expansion around $(\mu_X,\mu_Y)=(\E[\widehat{\mathrm{Num}}],\E[\widehat{\mathrm{Den}}])$ yields
\[
\E[f(\widehat{\mathrm{Num}},\widehat{\mathrm{Den}})]
\approx
\frac{\mu_X}{\mu_Y+\varepsilon}
\;+\;
\frac{1}{2}
\Big(
f_{xx}\Var(\widehat{\mathrm{Num}})
+2 f_{xy}\Cov(\widehat{\mathrm{Num}},\widehat{\mathrm{Den}})
+ f_{yy}\Var(\widehat{\mathrm{Den}})
\Big),
\]
with $f_x = 1/(\mu_Y+\varepsilon)$, $f_y=-\mu_X/(\mu_Y+\varepsilon)^2$, $f_{xx}=0$,
$f_{xy}=-1/(\mu_Y+\varepsilon)^2$, $f_{yy}=2\mu_X/(\mu_Y+\varepsilon)^3$.
Ignoring the mixed term (or if it is small by design—e.g., disjoint pairs), the dominant bias contribution is
\begin{equation}
\mathrm{Bias}
\;\approx\;
\E[\hat r_p^{\mathrm{reg}}]-r_p
\;\approx\;
-\,r_p \cdot \frac{\varepsilon_p}{\mu_Y}
\;+\;
r_p \cdot \frac{\Var(\widehat{\mathrm{Den}})}{(\mu_Y+\varepsilon_p)^2}
\;-\;
\frac{\Cov(\widehat{\mathrm{Num}},\widehat{\mathrm{Den}})}{(\mu_Y+\varepsilon_p)^2}.
\label{eq:delta-bias}
\end{equation}
Thus the shrinkage term $-\!r_p\,\varepsilon_p/\mu_Y$ is \emph{negative} and of the same order as the stochastic error; the variance term is positive and of order $1/N_{\mathrm{cov}}$; the cross-covariance term often reduces in absolute value under mean normalization and when we estimate $v_S$ via ridge-regularized Moebius regression.

\subsubsection{PARAFAC/CP}
For CP, $r = \frac{v_{\{p,q\}} (E[Y])^2}{v_{\{p\}} v_{\{q\}}}$ with two singletons in the denominator.
By the delta method in the \emph{log domain}, write
\[
\log \hat r
=
\big(\log \hat v_{\{p,q\}} - \log \hat v_{\{p\}} - \log \hat v_{\{q\}}\big)
+ \big(2\log \widehat{E}[Y]\big),
\]
so $\Var(\log \hat r)$ is a \emph{sum} of variances/covariances of these log-terms.
Translating back, the relative bias of \eqref{eq:reg-ratio} satisfies
\[
\frac{|\mathrm{Bias}|}{r}
= O\!\Big(\mathrm{CV}(\widehat{\mathrm{Den}})\Big)
= O\!\Big(\tfrac{1}{\sqrt{N_{\mathrm{cov}}}}\Big),
\]
where
$\mathrm{CV}(\widehat{\mathrm{Den}})=
\sqrt{
\mathrm{CV}^2(\hat v_{\{p\}})
+\mathrm{CV}^2(\hat v_{\{q\}})
+2\,\rho\,\mathrm{CV}(\hat v_{\{p\}})\mathrm{CV}(\hat v_{\{q\}})
}$ and $\rho$ is the correlation between the singleton estimates (empirically small if pairs are sampled independently across $S$). Mean normalization divides all $v_S$ by $(\widehat{E}[Y])^2$, shrinking both variance and covariance terms without changing the target.

\subsubsection{Tensor Train}
For interior bonds, $r_p = \frac{v_{\{p,p+1\}} v_{\{p-1,p+2\}}}{v_{\{p+1\}} v_{\{p-1,p,p+2\}}}$.
The denominator is a product; in the \emph{log domain} we have
\[
\log \hat r_p
=
\big(\log \hat v_{\{p,p+1\}} + \log \hat v_{\{p-1,p+2\}}\big)
-
\big(\log \hat v_{\{p+1\}} + \log \hat v_{\{p-1,p,p+2\}}\big),
\]
so
\(
\Var(\log \hat r_p)
\)
is a sum of the \emph{four} log-variances plus their cross-covariances.
Exponentiating back implies
\[
\frac{|\mathrm{Bias}|}{r_p}
\;=\;
O\!\Big(
\sqrt{
\sum_{S\in\mathcal{S}_p}
\Var\big(\log \hat v_{S}\big)
}
\Big)
\;=\;
O\!\Big(\tfrac{1}{\sqrt{N_{\mathrm{cov}}}}\Big),
\]
where $\mathcal{S}_p=\{\{p,p+1\},\,\{p-1,p+2\},\,\{p+1\},\,\{p-1,p,p+2\}\}$.
Mean normalization replaces $\hat v_S$ by $\tilde v_S$, reducing both $\Var(\log \hat v_S)$ and cross-covariances in practice because $(\widehat{E}[Y])^2$ is high-SNR. The regularization bias term in \eqref{eq:delta-bias} remains $O(1/\sqrt{N_{\mathrm{cov}}})$ since $\varepsilon_p \propto \mathrm{SE}(\widehat{\mathrm{Den}}_p)$.

\subsubsection{Tensor Ring}
For TR we estimate $\xi_p = \frac{(E[Y])^2 v_{\{p,p+1\}}}{v_{\{p\}} v_{\{p+1\}}}$ and solve the \emph{linear} system
$x_p + x_{p-1} - x_{p+1} = \psi_p$ with $x_p=\log r_p$ and $\psi_p=\log \hat \xi_p$.
Hence
\[
\widehat{\mathbf{x}}
=
\mathbf{C}^{-1}\widehat{\boldsymbol{\psi}},
\qquad
\Var(\widehat{\mathbf{x}})
=
\mathbf{C}^{-1}\Var(\widehat{\boldsymbol{\psi}})\big(\mathbf{C}^{-1}\big)^\top.
\]
Each $\psi_p$ is a \emph{log-ratio} of estimated moments; its variance is reduced by mean normalization (replacing $v_S$ by $\tilde v_S$), and any regularization is applied \emph{componentwise} to the underlying denominators that define $\xi_p$ (as in \eqref{eq:reg-ratio}). Since $\mathbf{C}$ is fixed and well-conditioned for $M\ge 3$, the bias order remains $O(1/\sqrt{N_{\mathrm{cov}}})$.

\subsubsection{General Guideline}
Across CP/TT/TR, with or without mean normalization, and with denominator regularization \eqref{eq:reg-ratio}, the induced bias satisfies
\[
\frac{|\mathrm{Bias}|}{r_p}
\;\lesssim\;
c\cdot \mathrm{CV}\big(\widehat{\mathrm{Den}}_p\big)
\;=\;
O\!\Big(\tfrac{1}{\sqrt{N_{\mathrm{cov}}}}\Big),
\]
for a constant $c$ that depends smoothly on the moment mix and (weak) cross-covariances. In practice, the reduction in variance (and hence RMSE) from the stabilizer $\varepsilon_p$ and from mean normalization outweighs the small $O(1/\sqrt{N_{\mathrm{cov}}})$ bias. We therefore recommend (i) mean normalization of $\hat v_S$ by $(\widehat{E}[Y])^2$; (ii) bootstrap-based $\varepsilon_p$; and (iii) (optionally) least-squares Moebius inversion with a small ridge penalty when computing $\hat v_S$, which further reduces both variance and cross-covariances among the pure terms.

\Section{Additional Experiments}\label{app:sec:further}

\paragraph{Tensor Train (\autoref{fig:appendix:tt_results}).}
Across settings, TT rank recovery tracks the $y{=}x$ line for small and moderate true ranks, with dispersion increasing as ranks grow or the signal-to-noise ratio (SNR) drops (e.g., smaller Poisson means or more dispersed Gamma priors). This is expected for moment-based estimators: the pure terms $v_S$ shrink in magnitude at low SNR, making inclusion–exclusion noisier and the denominators in the closed-form ratios closer to zero, which amplifies variance and occasionally induces slight upward bias. Interior bonds are typically better behaved than boundary bonds, which involve longer interaction sets. One practical mitigation that helped in our runs was \emph{mean normalization} of all $\hat v_S$ by $(\widehat{E}[Y])^2$, which reduces spread without changing the target. Across panels, estimation accuracy improves as the effective SNR increases:
low-mean or high-variance priors yield underestimation and wide dispersion,
while moderate $(\mu,\mathrm{cv})$ combinations (as predicted by the SNR sweet-spot analysis in Appendix~\ref{app:tt-variance-snr}) yield nearly unbiased recovery. The results confirm that the covariance-only TT estimator remains consistent when the prior mean and variance are chosen to keep $\mathrm{SNR}(S)\!\gtrsim\!5$ for the pure-term subsets involved in the ratio identities.

\begin{figure}[t]
    \centering
    \includegraphics[width=0.45\textwidth]{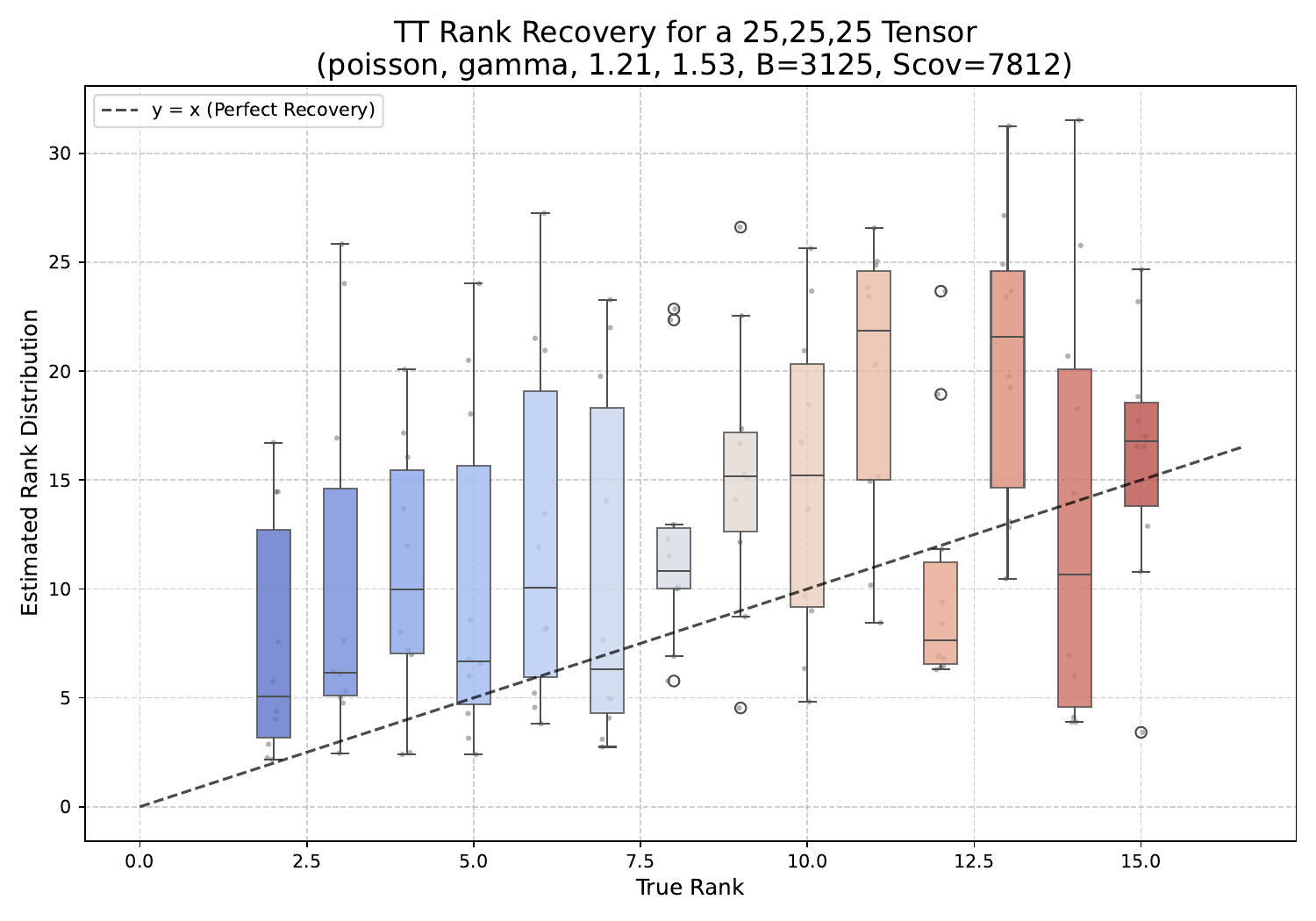} \hfill
    \includegraphics[width=0.45\textwidth]{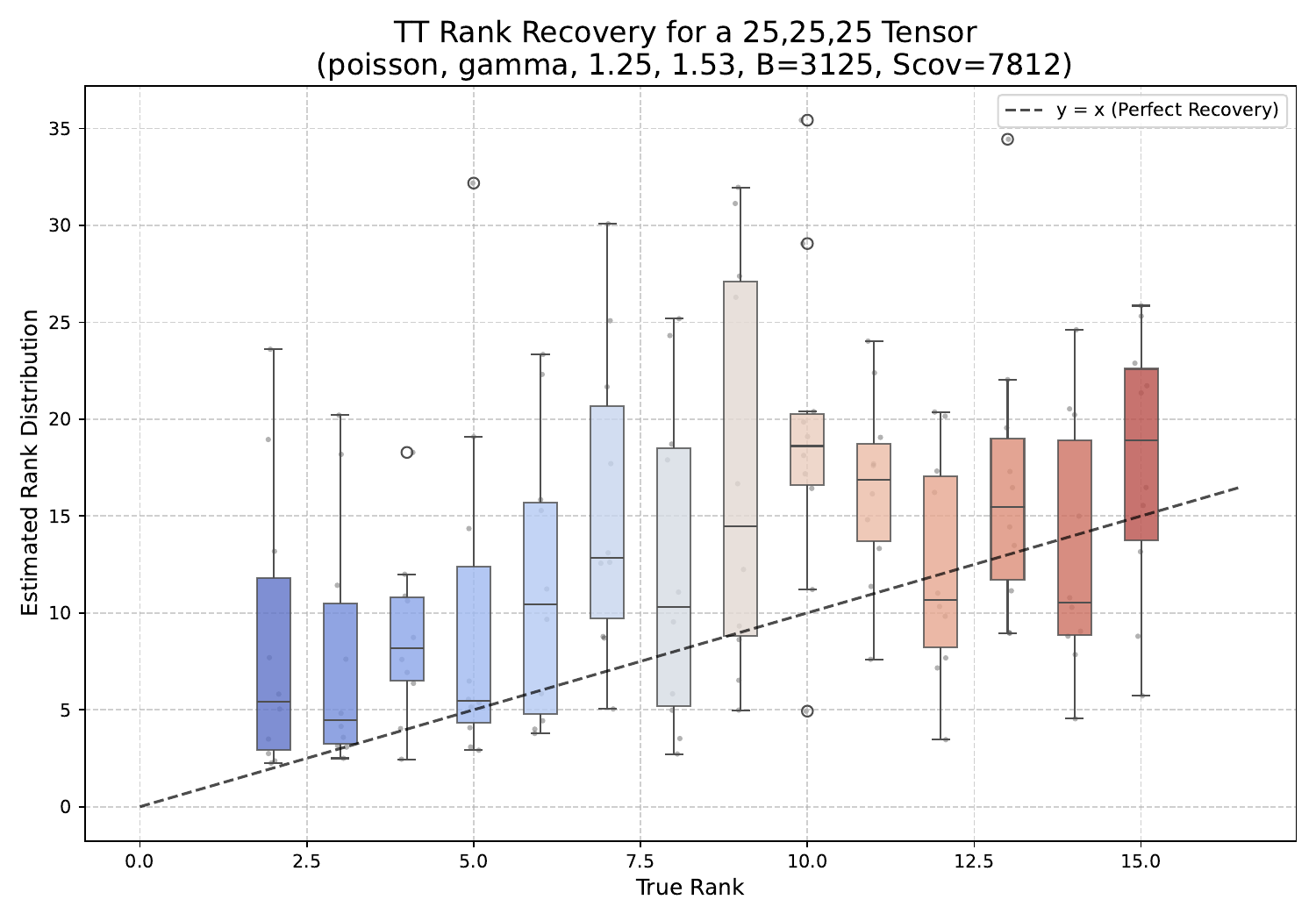} \hfill
    \includegraphics[width=0.45\textwidth]{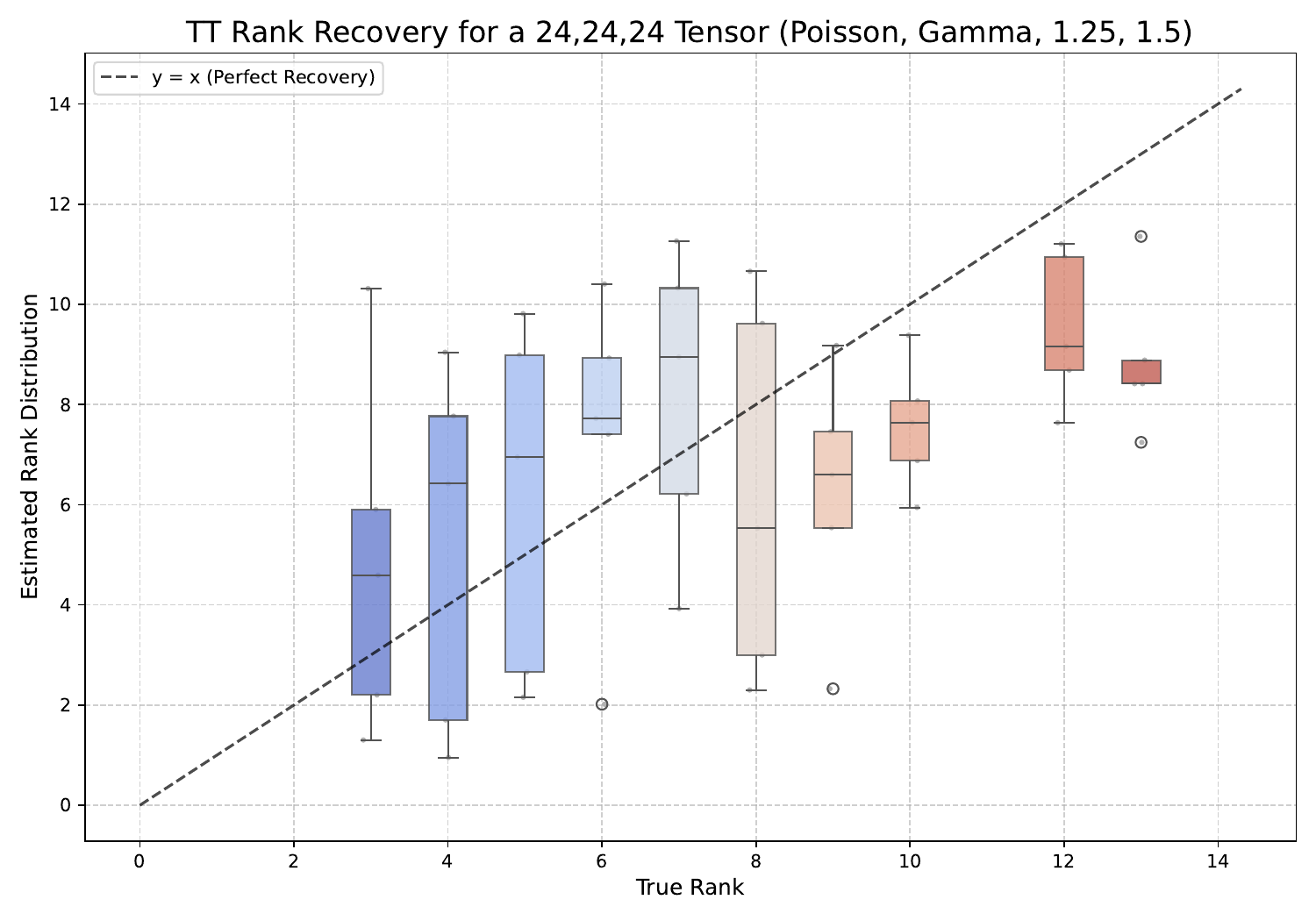} \hfill
    \includegraphics[width=0.45\textwidth]{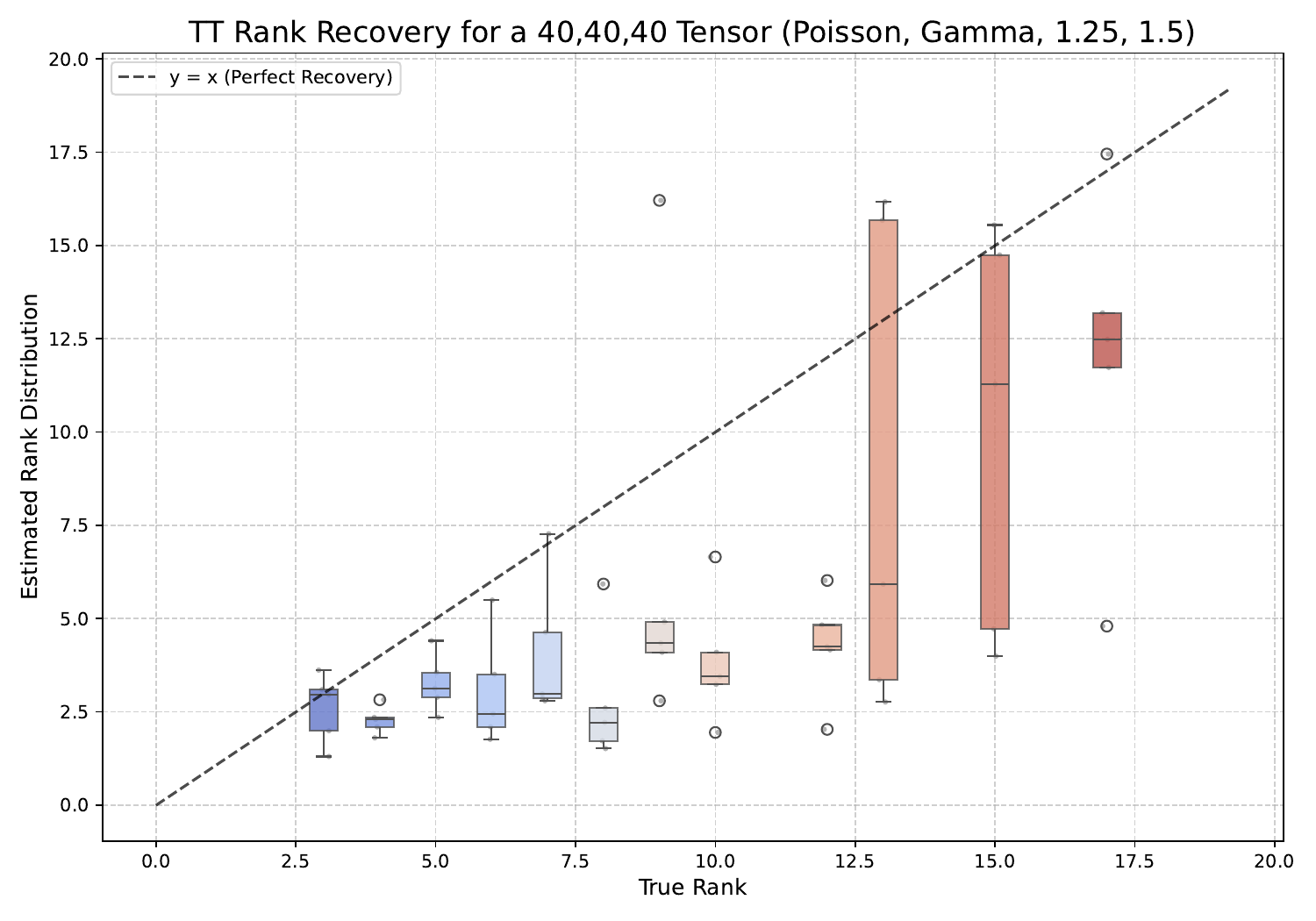} \hfill
    \includegraphics[width=0.45\textwidth]{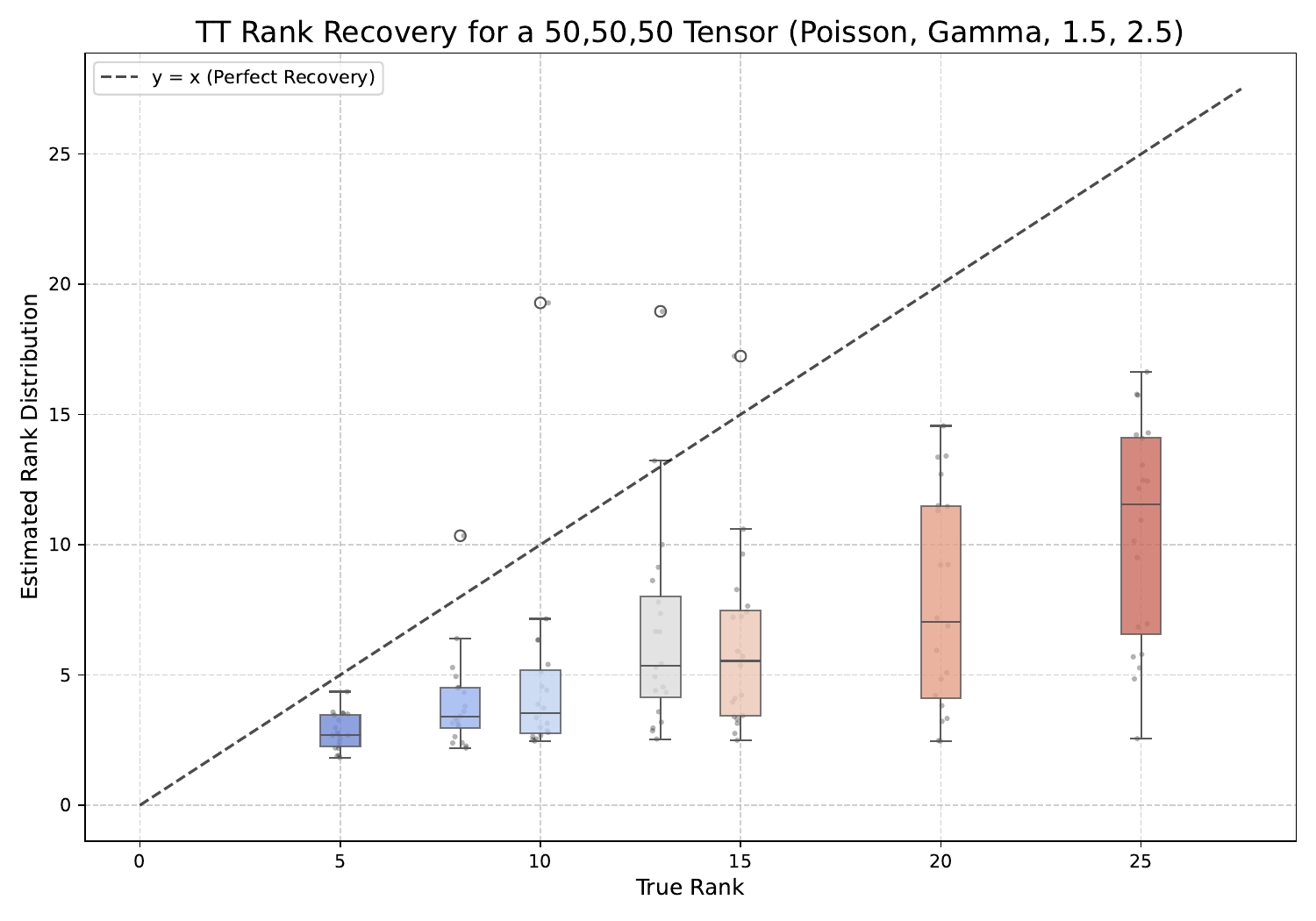} \hfill
    \caption{ Tensor Train (TT) rank estimation under varying dimensions and Gamma--Poisson generative settings. Each panel reports median rank estimates (over 20 independent runs) grouped by true rank, with boxplots showing the distribution of empirical estimates. The dashed diagonal ($y{=}x$) indicates perfect recovery. All models are trained using the covariance-only TT estimator. \textbf{Top row:} TT tensors of size $(25,25,25)$ with Poisson likelihood and Gamma priors
$\mathrm{Gamma}(\alpha,\theta)$ set to $(1.2,1.5)$ (left) and $(1.25,1.5)$ (right).
\textbf{Middle row:} TT tensors of sizes $(24,24,24)$ and $(40,40,40)$ with $\mathrm{Gamma}(1.25,1.5)$ priors, showing improved stability with increasing tensor size.
\textbf{Bottom:} TT tensor of size $(50,50,50)$ with $\mathrm{Gamma}(1.5,2.5)$ prior, corresponding to a higher mean and lower relative variance (lower $\mathrm{cv}$).
}
    \label{fig:appendix:tt_results}
\end{figure}

\paragraph{Tensor Ring (\autoref{fig:appendix:tr_results}).}
For TR, the median estimates follow the $y{=}x$ line across a wide range of ranks, with variance increasing at larger ranks and under lower SNR (smaller Poisson means / more dispersed Gamma priors). This behavior is consistent with our theory: each $\xi_p$ is a log–ratio of estimated moments, so noise in singletons $v_{\{p\}}$ and adjacent pairs $v_{\{p,p{+}1\}}$ propagates through the circulant linear system for $\{ \log r_p \}$. The cyclic coupling also means local errors can diffuse around the ring, slightly inflating uncertainty relative to TT at comparable ranks. In practice, \emph{mean normalization} of $\hat v_S$ by $(\widehat{E}[Y])^2$  tighten the spread of $\widehat{\xi}_p$ and, consequently, the recovered ranks.

\begin{figure}[ht!]
    \centering
    \includegraphics[width=0.45\textwidth]{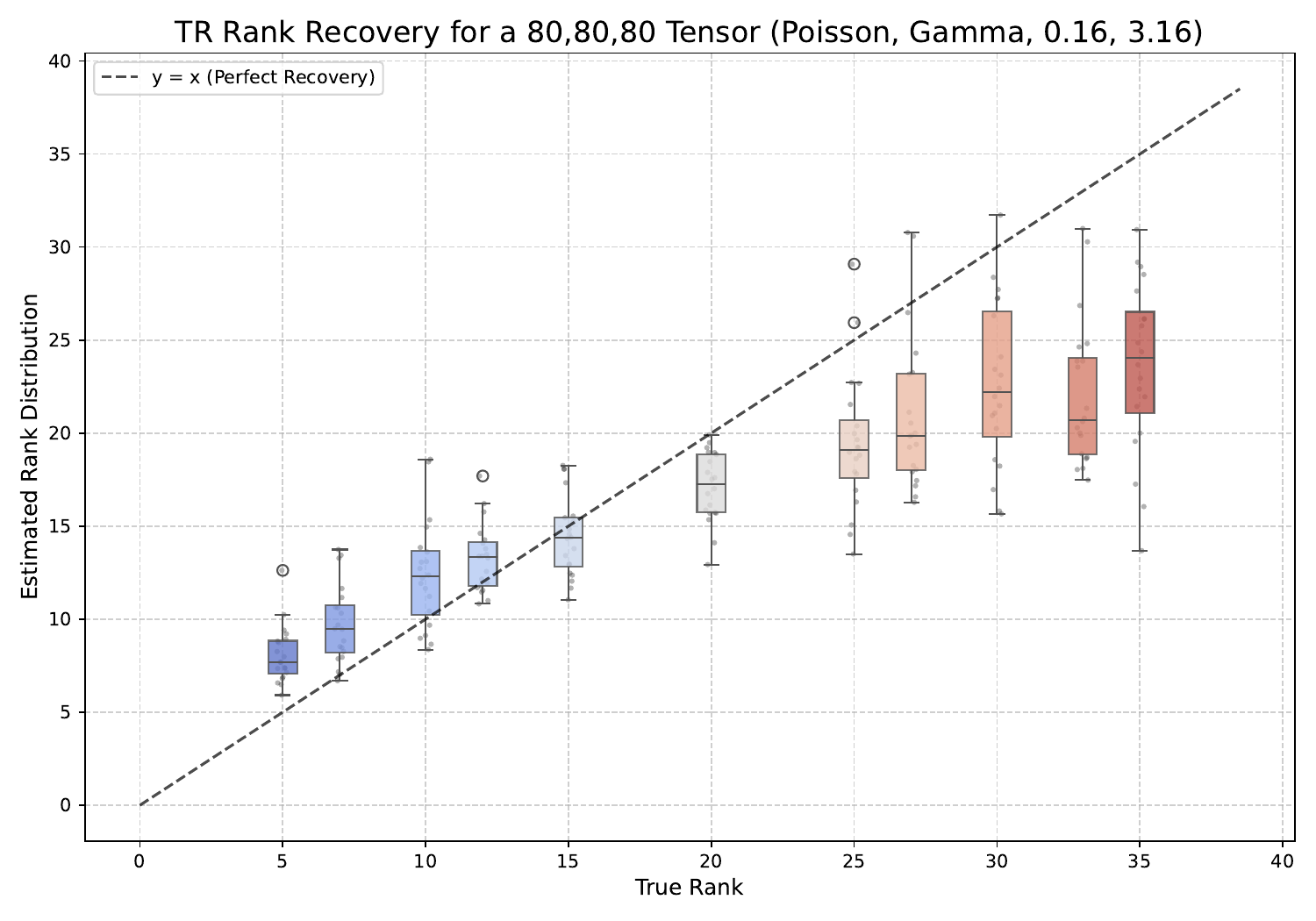} \hfill
    \includegraphics[width=0.45\textwidth]{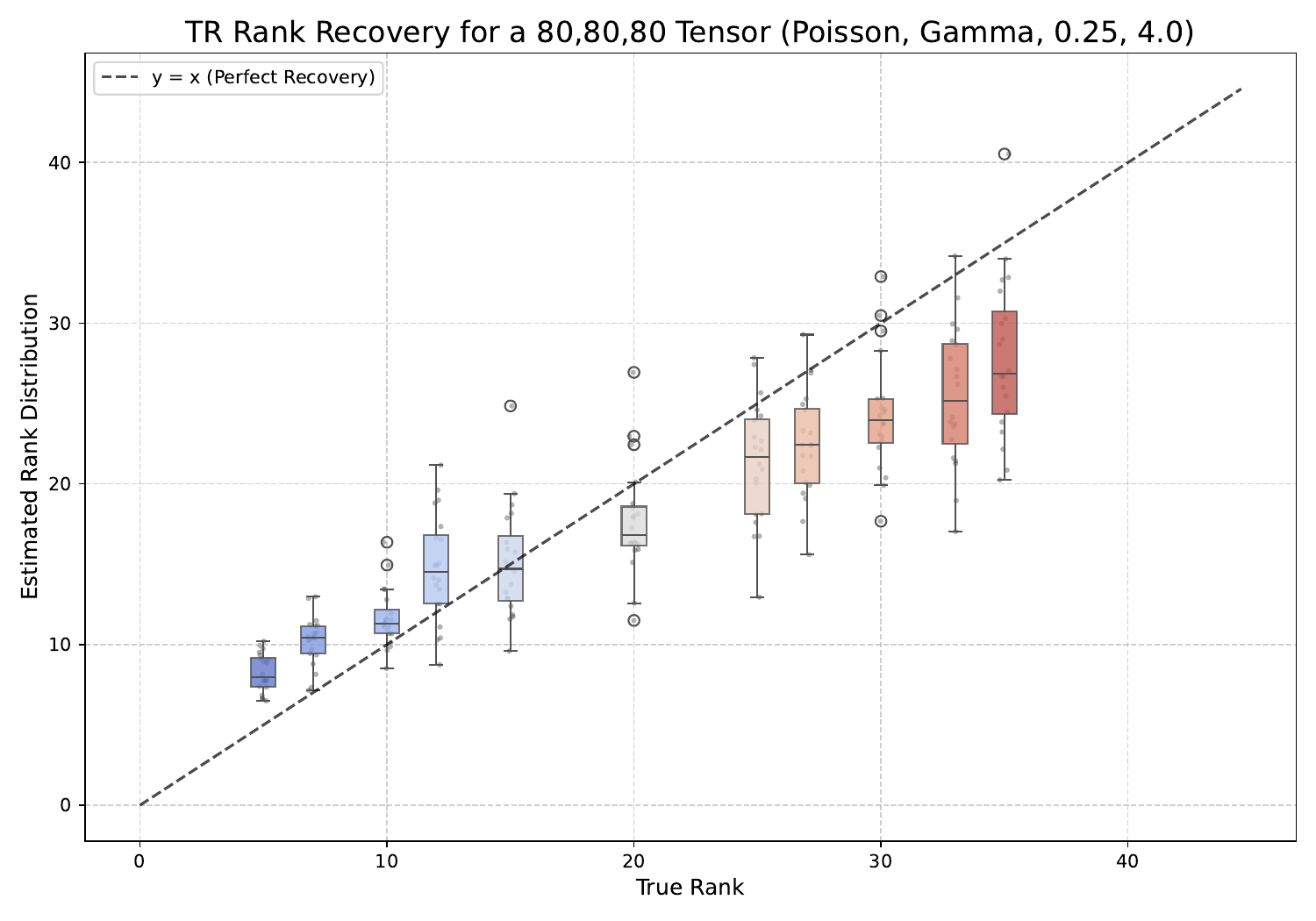}
    \caption{
\textbf{Tensor Ring (TR) rank estimation on $(80,80,80)$ tensors with Gamma--Poisson generative models.}
Each panel reports median rank estimates across 20 independent runs, grouped by true rank and visualized as boxplots.
The dashed diagonal ($y{=}x$) denotes perfect recovery.
Both experiments use a Poisson likelihood with Gamma priors $\mathrm{Gamma}(\alpha,\theta)$ given by
$(0.16,3.16)$ (left) and $(0.25,4.0)$ (right).
The shape and scale values correspond to different mean–variance tradeoffs, with the higher mean/lower variance prior (right)
yielding improved accuracy and reduced dispersion, consistent with the expected gain in effective signal-to-noise ratio.
}
    \label{fig:appendix:tr_results}
\end{figure}

\end{document}

%% file: math_commands.tex
\usepackage{amsmath,amsfonts,bm}

\def\eqref#1{equation~\ref{#1}}

\def\1{\bm{1}}

\DeclareMathAlphabet{\mathsfit}{\encodingdefault}{\sfdefault}{m}{sl}
\SetMathAlphabet{\mathsfit}{bold}{\encodingdefault}{\sfdefault}{bx}{n}

\newcommand{\E}{\mathbb{E}}

\newcommand{\Var}{\mathrm{Var}}

\newcommand{\Cov}{\mathrm{Cov}}


%% file: biblio.bib
@article{candes2013unbiased,
  title={Unbiased risk estimates for singular value thresholding and spectral estimators},
  author={Candes, Emmanuel J and Sing-Long, Carlos A and Trzasko, Joshua D},
  journal={IEEE transactions on signal processing},
  volume={61},
  number={19},
  pages={4643--4657},
  year={2013},
  publisher={IEEE}
}

@ARTICLE{10.3389/frai.2021.668353,
AUTHOR={Zhang, Kaiqi  and Hawkins, Cole  and Zhang, Zheng },       
TITLE={General-Purpose Bayesian Tensor Learning With Automatic Rank Determination and Uncertainty Quantification},   
JOURNAL={Frontiers in Artificial Intelligence},    
VOLUME={Volume 4 - 2021},
YEAR={2022},
URL={https://www.frontiersin.org/journals/artificial-intelligence/articles/10.3389/frai.2021.668353},
DOI={10.3389/frai.2021.668353}}

@article{
zhu2025prior,
title={Prior Specification for Exposure-based Bayesian Matrix Factorization},
author={Zicong Zhu and Issei Sato},
journal={Transactions on Machine Learning Research},
issn={2835-8856},
year={2025},
url={https://openreview.net/forum?id=o5R4Hv9XqC},
note={}
}

@article{DBLP:journals/siamrev/KoldaB09,
  author       = {Tamara G. Kolda and
                  Brett W. Bader},
  title        = {Tensor Decompositions and Applications},
  journal      = {{SIAM} Rev.},
  volume       = {51},
  number       = {3},
  pages        = {455--500},
  year         = {2009}
}

@article{hitchcock1927,
author = {Hitchcock, Frank L.},
title = {The Expression of a Tensor or a Polyadic as a Sum of Products},
journal = {Journal of Mathematics and Physics},
volume = {6},
number = {1-4},
pages = {164-189},
doi = {https://doi.org/10.1002/sapm192761164},
year = {1927}
}

@article{tucker1966some_tucker,
  title={Some mathematical notes on three-mode factor analysis},
  author={Tucker, Ledyard R},
  journal={Psychometrika},
  volume={31},
  number={3},
  pages={279--311},
  year={1966},
  publisher={Springer}
}

@article{oseledets2011tensor_tt,
  title={Tensor-train decomposition},
  author={Oseledets, Ivan V},
  journal={SIAM Journal on Scientific Computing},
  volume={33},
  number={5},
  pages={2295--2317},
  year={2011},
  publisher={SIAM}
}

@inproceedings{zhao2016_tr,
  title={Tensor Ring Decomposition},
  author={Zhao, Qibin and Zhou, Guoxu and Xie, Shengli and Zhang, Lihua and Cichocki, Andrzej},
  booktitle={2016 24th European Signal Processing Conference (EUSIPCO)},
  pages={2280--2284},
  year={2016},
  organization={IEEE},
  doi={10.1109/EUSIPCO.2016.7760655}
}

@inproceedings{tensornet_supervised,
author = {Stoudenmire, E. M. and Schwab, David J.},
title = {Supervised learning with tensor networks},
year = {2016},
booktitle = {NeurIPS},
pages = {4806–4814},
numpages = {9},
location = {Barcelona, Spain},
series = {NIPS'16}
}

@article{harshman1970foundations_parafac,
  title={Foundations of the PARAFAC procedure: Models and conditions for an “explanatory” multi-modal factor analysis},
  author={Harshman, Richard A and others},
  journal={UCLA working papers in phonetics},
  volume={16},
  number={1},
  pages={84},
  year={1970},
  publisher={Los Angeles, CA}
}

@inproceedings{DBLP:journals/jmlr/ChuG09,
  author       = {Wei Chu and
                  Zoubin Ghahramani},
  title        = {Probabilistic Models for Incomplete Multi-dimensional Arrays},
  booktitle    = {{AISTATS}},
  series       = {{JMLR} Proceedings},
  volume       = {5},
  pages        = {89--96},
  publisher    = {JMLR.org},
  year         = {2009}
}

@article{contisciani2022inference,
  title={Inference of hyperedges and overlapping communities in hypergraphs},
  author={Contisciani, Martina and Battiston, Federico and De Bacco, Caterina},
  journal={Nature communications},
  volume={13},
  number={1},
  pages={7229},
  year={2022},
  publisher={Nature Publishing Group UK London}
}

@inproceedings{DBLP:conf/icml/RaiWGCDC14,
  author       = {Piyush Rai and
                  Yingjian Wang and
                  Shengbo Guo and
                  Gary Chen and
                  David B. Dunson and
                  Lawrence Carin},
  title        = {Scalable Bayesian Low-Rank Decomposition of Incomplete Multiway Tensors},
  booktitle    = {ICML},
  series       = {{JMLR} Workshop and Conference Proceedings},
  volume       = {32},
  pages        = {1800--1808},
  publisher    = {JMLR.org},
  year         = {2014}
}

@inproceedings{DBLP:conf/aaai/PorteousBW08,
  author       = {Ian Porteous and
                  Evgeniy Bart and
                  Max Welling},
  title        = {Multi-HDP: {A} Non Parametric Bayesian Model for Tensor Factorization},
  booktitle    = {AAAI},
  pages        = {1487--1490},
  publisher    = {{AAAI} Press},
  year         = {2008}
}

@article{alter2003generalized,
  title={Generalized singular value decomposition for comparative analysis of genome-scale expression data sets of two different organisms},
  author={Alter, Orly and Brown, Patrick O and Botstein, David},
  journal={Proceedings of the National Academy of Sciences},
  volume={100},
  number={6},
  pages={3351--3356},
  year={2003},
  publisher={National Academy of Sciences}
}

@inproceedings{DBLP:conf/neurips/NovikovPOV15/tensorizing,
  author       = {Alexander Novikov and
                  Dmitry Podoprikhin and
                  Anton Osokin and
                  Dmitry P. Vetrov},
  title        = {Tensorizing Neural Networks},
  booktitle    = {{NeurIPS}},
  pages        = {442--450},
  year         = {2015}
}

@article{dasilva2023prior,
  author  = {da Silva, Eliezer de Souza and Kuśmierczyk, Tomasz and Hartmann, Marcelo  and Klami, Arto },
  title   = {Prior Specification for Bayesian Matrix Factorization via Prior Predictive Matching},
  journal = {Journal of Machine Learning Research},
  year    = {2023},
  volume  = {24},
  number  = {67},
  pages   = {1--51},
  url     = {http://jmlr.org/papers/v24/21-0623.html}
}

@article{DBLP:journals/jal/Hastad90,
  author       = {Johan H{\aa}stad},
  title        = {Tensor Rank is NP-Complete},
  journal      = {J. Algorithms},
  volume       = {11},
  number       = {4},
  pages        = {644--654},
  year         = {1990}
}

@article{DBLP:journals/jmlr/WangS21,
  author       = {Wei Wang and
                  Matthew Stephens},
  title        = {Empirical Bayes Matrix Factorization},
  journal      = {J. Mach. Learn. Res.},
  volume       = {22},
  pages        = {120:1--120:40},
  year         = {2021}
}

@article{DBLP:journals/pami/ZhaoZC15,
  author       = {Qibin Zhao and
                  Liqing Zhang and
                  Andrzej Cichocki},
  title        = {Bayesian {CP} Factorization of Incomplete Tensors with Automatic Rank
                  Determination},
  journal      = {{IEEE} Trans. Pattern Anal. Mach. Intell.},
  volume       = {37},
  number       = {9},
  pages        = {1751--1763},
  year         = {2015}
}

@inproceedings{DBLP:conf/www/BhargavaPZL15,
  author       = {Preeti Bhargava and
                  Thomas Phan and
                  Jiayu Zhou and
                  Juhan Lee},
  title        = {Who, What, When, and Where: Multi-Dimensional Collaborative Recommendations
                  Using Tensor Factorization on Sparse User-Generated Data},
  booktitle    = {{WWW}},
  pages        = {130--140},
  publisher    = {{ACM}},
  year         = {2015}
}

@article{DBLP:journals/bspc/MosayebiH20,
  author       = {Raziyeh Mosayebi and
                  Gholam{-}Ali Hossein{-}Zadeh},
  title        = {Correlated coupled matrix tensor factorization method for simultaneous
                  EEG-fMRI data fusion},
  journal      = {Biomed. Signal Process. Control.},
  volume       = {62},
  pages        = {102071},
  year         = {2020}
}

@inproceedings{gopalan2015scalable,
  title={Scalable recommendation with hierarchical Poisson factorization},
  author={Gopalan, Prem and Hofman, Jake M and Blei, David M},
   booktitle = {Proceedings of the Thirty-First Conference on Uncertainty in Artificial Intelligence},
   series = {UAI'15},
  year={2015},
}

@inproceedings{DBLP:conf/kdd/ScheinPBW15,
  author       = {Aaron Schein and
                  John W. Paisley and
                  David M. Blei and
                  Hanna M. Wallach},
  title        = {Bayesian Poisson Tensor Factorization for Inferring Multilateral Relations
                  from Sparse Dyadic Event Counts},
  booktitle    = {KDD},
  pages        = {1045--1054},
  publisher    = {{ACM}},
  year         = {2015}
}

@article{cpcande70,
  author = {Carroll, J. and Chang, J.},
  journal = {Psychometrika},
  pages = {283--319},
  title = {Analysis of individual differences in multidimensional scaling via an n-way generalization of eckart-young decomposition},
  year = 1970
}

@article{hubsspokenetwork,
author = {Carlsson, John Gunnar and Jia, Fan},
title = {Euclidean Hub-and-Spoke Networks},
journal = {Operations Research},
volume = {61},
number = {6},
pages = {1360-1382},
year = {2013},
doi = {10.1287/opre.2013.1219}
}

@article{tensornet_ftml,
year = {2016},
volume = {9},
journal = {Foundations and Trends® in Machine Learning},
title = {Tensor Networks for Dimensionality Reduction and Large-scale Optimization: Part 1 Low-Rank Tensor Decompositions},
doi = {10.1561/2200000059},
issn = {1935-8237},
number = {4-5},
pages = {249-429},
author = {Andrzej Cichocki and Namgil Lee and Ivan Oseledets and Anh-Huy Phan and Qibin Zhao and Danilo P. Mandic}
}

@book{cichocki2009nonnegative,
  title={Nonnegative matrix and tensor factorizations: applications to exploratory multi-way data analysis and blind source separation},
  author={Cichocki, Andrzej and Zdunek, Rafal and Phan, Anh Huy and Amari, Shun-ichi},
  year={2009},
  publisher={John Wiley \& Sons}
}

@article{brochem2003,
author = {Bro, Rasmus and Kiers, Henk A. L.},
title = {A new efficient method for determining the number of components in PARAFAC models},
journal = {Journal of Chemometrics},
volume = {17},
number = {5},
pages = {274-286},
doi = {https://doi.org/10.1002/cem.801},
year = {2003}
}

@article{hillartensorhard2013,
author = {Hillar, Christopher J. and Lim, Lek-Heng},
title = {Most Tensor Problems Are NP-Hard},
year = {2013},
issue_date = {November 2013},
publisher = {Association for Computing Machinery},
address = {New York, NY, USA},
volume = {60},
number = {6},
issn = {0004-5411},
url = {https://doi.org/10.1145/2512329},
doi = {10.1145/2512329},
journal = {J. ACM},
month = nov,
articleno = {45},
numpages = {39}
}

@inproceedings{DBLP:conf/approx/Swernofsky18,
  author       = {Joseph Swernofsky},
  title        = {Tensor Rank is Hard to Approximate},
  booktitle    = {{APPROX-RANDOM}},
  series       = {LIPIcs},
  volume       = {116},
  pages        = {26:1--26:9},
  publisher    = {Schloss Dagstuhl - Leibniz-Zentrum f{\"{u}}r Informatik},
  year         = {2018}
}

@article{cong_eeg2015,
  title={Tensor decomposition of EEG signals: a brief review},
  author={Cong, Fengyu and Lin, Qiu-Hua and Kuang, Li-Dan and Gong, Xiao-Feng and Astikainen, Piia and Ristaniemi, Tapani},
  journal={Journal of neuroscience methods},
  volume={248},
  pages={59--69},
  year={2015},
  publisher={Elsevier}
}
